%% file: main.tex
\documentclass[11pt]{article}
\input{math_commands.tex}

\newcommand{\pos}{{\mathbf{pos}}}
\newcommand{\ctx}{{\mathbf{ctx}}}
\newcommand{\resid}{{\mathbf{resid}}}
\newcommand{\res}{{\mathrm{res}}}
\newcommand{\cvec}{{\mathbf{cvec}}}

\newcommand{\head}{{\mathrm{head}}}
\newcommand{\incoh}{{\mathrm{incoh}}}

\newcommand{\Embed}{{\texttt{Embed}}}

\newcommand{\TFLayer}{{\texttt{TFLayer}}}
\newcommand{\LN}{{\texttt{LN}}}
\newcommand{\FFN}{{\texttt{FFN}}}
\newcommand{\MHA}{{\texttt{MHA}}}

\newcommand{\AttnHead}{{\mathrm{AttnHead}}}
\newcommand{\diagg}{{\mathrm{diagg}}}
\newcommand{\spann}{{\mathrm{span}}}
\makeatletter
\newcommand*{\rom}[1]{\expandafter\@slowromancap\romannumeral #1@}
\makeatother

\usepackage{amsmath,amssymb,amsfonts,amsthm,bbm}
\usepackage{times}
\RequirePackage[colorlinks,citecolor=blue,urlcolor=blue,linkcolor=blue,linktocpage=true]{hyperref}
\usepackage{url}
\usepackage{graphicx}
\usepackage{multicol}
\usepackage{multirow}
\usepackage{array}
\usepackage{hhline}
\usepackage{wrapfig,lipsum}
\usepackage{soul}
\usepackage{booktabs}
\usepackage{natbib}
\usepackage{verbatim}
\usepackage{setspace}
\usepackage{graphicx}
\usepackage{subcaption}
\usepackage{enumitem}
\usepackage[dvipsnames]{xcolor}
\usepackage{bbm}
\usepackage[top=1in, bottom=1in, left=1in, right=1in]{geometry}
\usepackage[toc,page,header]{appendix}
\usepackage[]{minitoc}

\noptcrule 

\newcolumntype{C}[1]{>{\centering\arraybackslash}m{#1}}

\title{Uncovering hidden geometry in Transformers via disentangling position and context
}

\author{
Jiajun Song\thanks{National Key Laboratory of General Artificial Intelligence,
Beijing Institute for General Artificial Intelligence (BIGAI),
Beijing 100080, China,
\texttt{songjiajun@bigai.ai}}, \;\; 
Yiqiao Zhong\thanks{
Department of Statistics,
University of Wisconsin--Madison,
WI, 53706, USA,
\texttt{yiqiao.zhong@wisc.edu}} 
}

\begin{document}

\maketitle







\begin{abstract}
Transformers are widely used to extract semantic meanings from input tokens, yet they usually operate as black-box models. In this paper, we present a simple yet informative decomposition of hidden states (or embeddings) of trained transformers into interpretable components. 
For any layer, embedding vectors of input sequence samples are represented by a tensor $\vh \in \R^{C \times T \times d}$. Given embedding vector $\vh_{c,t} \in \R^d$ at sequence position $t \le T$ in 
a sequence (or context) $c \le C$, extracting the mean effects yields the decomposition
\[
\vh_{c,t} = \vmu + \pos_t + \ctx_c + \resid_{c,t}
\]
where $\vmu$ is the global mean vector, $\pos_t$ and $\ctx_c$ are the mean vectors across contexts and across positions respectively, and $\resid_{c,t}$ is the residual vector. For popular transformer architectures and diverse text datasets, empirically we find pervasive mathematical structure: (1) $(\pos_t)_{t}$ forms a low-dimensional, continuous, and often spiral shape across layers, (2) $(\ctx_c)_c$ shows clear cluster structure that falls into context topics, and (3) $(\pos_t)_{t}$ and $(\ctx_c)_c$ are 
nearly orthogonal.
We argue that smoothness is pervasive and beneficial to transformers trained on languages, and our decomposition leads to improved model interpretability.
\end{abstract}

\doparttoc 
\faketableofcontents 
\part{} 

\section{Introduction}
\label{sec:intro}

Transformers \citep{vaswani2017attention} are practical neural network models that underlie recent successes of large language models (LLMs) \citep{brown2020language, bubeck2023sparks}. 
Unfortunately, transformers are often used as black-box models due to lack of in-depth analyses of 
 internal mechanism, which raises concerns such as lack of interpretability, model biases, security issues, etc., \citep{bommasani2021opportunities}. 
 
In particular, it is poorly understood what information embeddings from each layer capture. We identify two desiderata: (1) internal quantitative measurements, particularly for the intermediate layers; (2) visualization tools and diagnostics tailored to transformers beyond attention matrix plots. 

\begin{figure*}[t!]
\centering
\includegraphics[width=0.98\textwidth]
{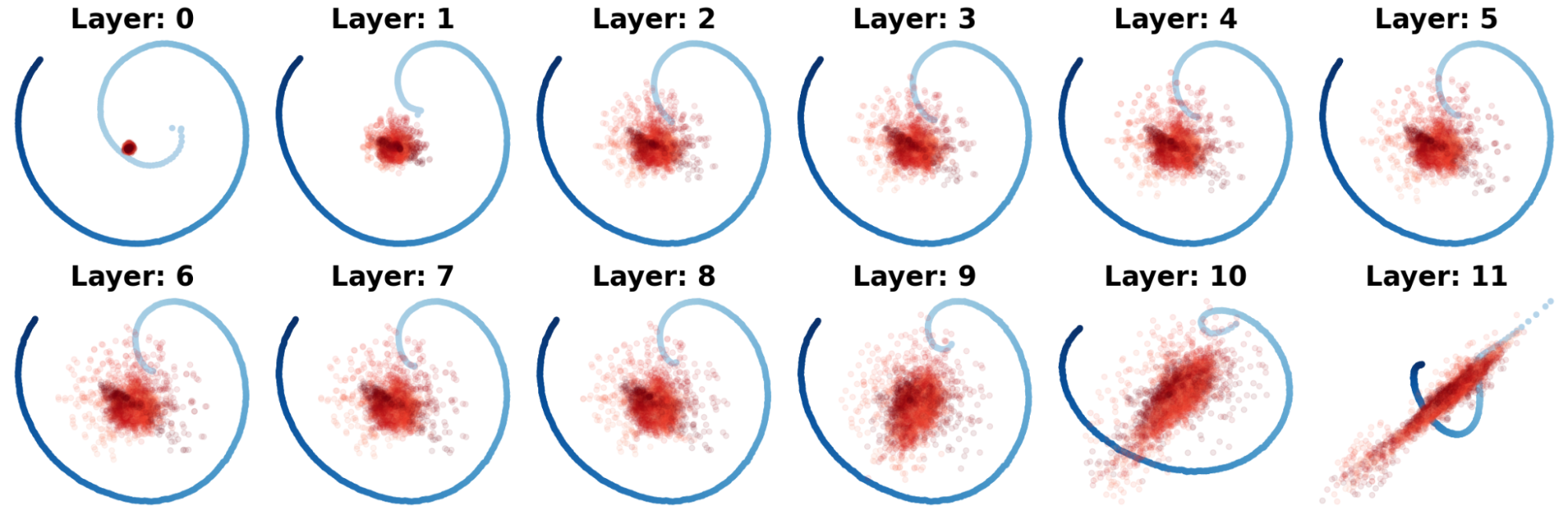}
\caption{\textbf{PCA visualization} of positional basis ({\color{blue}blue}) and cvecs ({\color{red} red}) from GPT-2 on OpenWebText. For every layer $\ell$, each $\pos_t^{(\ell)}$ and randomly selected $\cvec_{c,t}^{(\ell)}$ 
are projected using top-$2$ principal directions of $(\pos_t^{(\ell)})_{t \le T}$. Darker blue/red colors correspond to larger $t$. Principal components have dramatically \textbf{increasing scales} across layers, but for aesthetic purposes we rescaled all plots.}\label{fig:PCA}
\end{figure*}

Let us introduce basic notations. An input sequence consists of $T$ consecutive tokens (e.g., words or subwords), and a corpus is a collection of all input sequences. Let $C$ be the total number of input sequences and $c \le C$ denote a generic sequence, which may be represented by  $\vx_{c,1},\ldots,\vx_{c,T}$ where each $\vx_{c,t}$ corresponds to a token. We start from the initial static (and positional) embeddings $(\vh_{c,t}^{(0)})_{t \le T}$ and then calculate the intermediate-layer embeddings $(\vh_{c,t}^{(\ell)})_{t \le T}$:
\begin{align*}
&\vh^{(0)}_{c,1}, \ldots, \vh^{(0)}_{c,T} = \Embed(\vx_{c,1}, \ldots, \vx_{c,T}) \\
&\vh^{(\ell)}_{c,1}, \ldots, 
\vh^{(\ell)}_{c,T} = \TFLayer_\ell(\vh^{(\ell-1)}_{c,1}, \ldots, \vh^{(\ell-1)}_{c,T}),
\end{align*}
for $\ell=1,\ldots,L$, where $\Embed$ and $\TFLayer_\ell$ are general mappings. This general definition encompasses many transformer models, which depend on attention heads defined as follows. Given $d_\head \le d$ and input matrix $\mX \in \R^{T \times d}$, for trainable weights $\mW^q, \mW^k, \mW^v \in \R^{d \times d_\head}$, define
\begin{equation}\label{AttnHead}
\AttnHead(\mX) = \softmax\left(\frac{\mX \mW^q (\mW^k)^\top \mX^\top} {\sqrt{d_{\head}}} \right)\mX \mW^v.
\end{equation}
Multi-head attention heads, denoted by $\MHA$, are essentially the concatenation of many attention heads. Denote a generic fully-connected layer by $\FFN(\vx) = \mW_2\max\{\vzero, \mW_1 \vx + \vb_1\} + \vb_2$ given any $\vx\in \R^d$ for trainable weights $\mW_1 \in \R^{d'\times d}, \mW_2\in\R^{d\times d'},  \vb_1 \in \R^{d'},\vb_2 \in \R^d$ (often $d'=4d$), and let $\LN$ be a generic layer normalization layer. The standard transformer is expressed as
\begin{align*}
&\vh_c^{(\ell+0.5)} = \vh_c^{(\ell)} + \MHA^{(\ell)}(\LN^{(\ell,1)}(\vh_c^{(\ell)})), \\
&\vh_{c}^{(\ell+1)} = \vh_{c}^{(\ell+0.5)} + \FFN^{(\ell)}(\LN^{(\ell,2)}((\vh_{c}^{(\ell+0.5)}))) 
\end{align*}
where $\vh_c^{(\ell+0.5)} = (\vh_{c,1}^{(\ell+0.5)},\ldots,\vh_{c,T}^{(\ell+0.5)})$ and $\vh_c^{(\ell)} = (\vh_{c,1}^{(\ell)},\ldots,\vh_{c,T}^{(\ell)})$.

\subsection{A mean-based decomposition}

For each embedding vector $\vh^{(\ell)}_{c,t} \in \R^d$ from any trained transformer, consider the decomposition
\begin{equation}
\vh_{c,t}^{(\ell)} = \vmu^{(\ell)} + \pos_t^{(\ell)} + \ctx_c^{(\ell)} + \resid_{c,t}^{(\ell)},  \label{decomp}
\end{equation}
\begin{align}
    &\vmu^{(\ell)} := \frac{1}{CT} \sum_{c,t} \vh_{c,t}^{(\ell)}, ~\pos_t^{(\ell)} := \frac{1}{C} \sum_c \vh_{c,t}^{(\ell)} - \vmu^{(\ell)}, \label{def:components} \\
    &\ctx_c^{(\ell)} := \frac{1}{T} \sum_t \vh_{c,t}^{(\ell)} - \vmu^{(\ell)}\, .  \label{def:components2}
\end{align}
Each of the four components has the following interpretations. For any given layer $\ell$,
\begin{itemize}
    \item we call $\vmu^{(\ell)}$ the global mean vector, which differentiates neither contexts nor positions;
    \item we call $(\pos_t^{(\ell)})_{t\le T}$ the positional basis, as they quantify average positional effects;
    \item we call $(\ctx_c^{(\ell)})_{c\le C}$ the context basis, as they quantify average sequence/context effects;
    \item we call $(\resid_{c,t}^{(\ell)})_{t\le T, c \le C}$ the residual vectors, which capture higher-order effects.
    \item In addition, we define $\cvec_{c,t}^{(\ell)} = \ctx_c^{(\ell)} + \resid_{c,t}^{(\ell)}$.
\end{itemize}
A corpus may contain billions of tokens. For practical use, in this paper $C$ is much smaller: we subsample input sequences from the corpus; for example, $C=6.4K$ in Figure~\ref{fig:PCA}. 

\paragraph{Positional basis vs.~positional embeddings.} While positional embeddings at Layer 0 is much explored in the literature (see Section~\ref{sec:related}), the structure in intermediate layers is poorly understood. In contrast, our approach offers structural insights to \textit{all} layers.

\subsection{Connections to ANOVA}
Our embedding decomposition is similar to multivariate two-way ANOVA in form. Borrowing standard terminology from ANOVA, positions and contexts can be regarded as two \textit{factors} or \textit{treatments}, so viewing the embedding $\vh_{c,t}$ as the response variable, then positional/context bases represent mean effects.

\paragraph{On terminology.} (i) We use \textit{context} to refer to a sequence since its tokens collectively encode context information. (ii) We call positional/context basis for convenience. A more accurate term is \textit{frame} or \textit{overcomplete basis}, since $(\pos_t^{(\ell)})_{t\le T}$ and $(\ctx_t^{(\ell)})_{c\le C}$ are often linearly dependent.

\subsection{Accessible reproducibility} 
We provide a fast implementation via Google Colab that reproduces most of the figures and analysis for GPT-2 (under several minutes with the GPU option):
\begin{center}
{\footnotesize \url{https://colab.research.google.com/drive/1ubsJQvLkOSQtiU8LoBA_79t1bd5-5ihi?usp=sharing} }.
\end{center}
The complete implementation, as well as additional plots and measurements, can be found on the following GitHub page.
\begin{center}
\url{https://github.com/JiajunSong629/uncover-hidden-geometry}
\end{center}

\subsection{Notations}
For a vector $\vx$, we denote its $\ell_2$ norm by $\norm{\vx}$. For a matrix $\mA \in \R^{n \times m}$, we denote its operator norm by $\norm{\mA}_{\mathrm{op}} := \max_{\vu: \norm{\vu}=1}\norm{\mA \vu}$ and max norm by $\norm{\mA}_{\mathrm{max}} := \max_{i,j}|A_{ij}|$. We use the standard big-$O$ notation: for positive scalars $a,b$, we write $a = O(b)$ if $a \le C b$ for a constant $C$. We use $\spann(\mM)$ to denote the linear span of column vectors in $\mM$.

\begin{table*}[t!]
\caption{\textbf{Averaged (and std of) measurements across layers.} Measurements based on 6.4K samples. All values are in $[0,1]$ except `rank estimate': `relative norm' means magnitude of positional basis relative to centered embeddings; `similarity' and `incoherence' are averaged \textit{cosine similarity} (inner products of normalized vectors) between $\ctx$, and between $\ctx$ and $\pos$, respectively. We find (i) positional basis is a \textbf{low-rank} and \textbf{significant} component (ii) inter similarity $\mathbf{\ll}$ intra similarity (iii) \textbf{incoherence} between two bases.}\label{tab:main}
\begin{center}
\begin{tabular}{|C{1.2cm}|C{1.8cm}|C{2cm}|C{2cm}|C{2cm}|C{2cm}|C{2.3cm}|}
\hline
 \multicolumn{2}{|c|}{  } & \multicolumn{2}{c|}{\centering Positional basis} & \multicolumn{2}{c|}{\centering Context basis} & \multirow{2}{*}{ \vspace{-0.5cm}Incoherence} \\
 \cline{3-6} 
 \multicolumn{2}{|c|}{ } & rank estimate & relative norm & inter-cluster similarity & intra-cluster similarity&  \\

\hline
\footnotesize{NanoGPT} & \footnotesize{Shakespeare} & 
 \footnotesize{$\phantom{0}7.86~(1.96)$} & \footnotesize{$0.66~(0.28)$} & \footnotesize{--------} & \footnotesize{--------} &  \footnotesize{--------}  \\

 \hline
\multirow{2}{*}{{ \footnotesize{GPT-2}}} & \footnotesize{OpenWebText} & 
 \footnotesize{$11.38~(1.86)$} & \footnotesize{$0.84~(0.18)$} & \footnotesize{$0.10~(0.01)$} & \footnotesize{$0.44~(0.04)$} & \footnotesize{$0.051~(0.05)$} \\
\cline{2-7}
 & \footnotesize{WikiText} & \footnotesize{$11.69~(1.64)$} & \footnotesize{$0.76~(0.18)$} & \footnotesize{$0.11~(0.01)$} & \footnotesize{$0.41~(0.03)$} & \footnotesize{$0.039~(0.04)$}  \\

 \hline

\multirow{2}{*}{{ \footnotesize{BERT}}} & \footnotesize{OpenWebText} & 
 \footnotesize{$12.54~(2.73)$} & \footnotesize{$0.78~(0.18)$} & \footnotesize{$0.13~(0.04)$} & \footnotesize{$0.26~(0.04)$} & \footnotesize{$0.046~(0.05)$} \\
\cline{2-7}
 & \footnotesize{WikiText} & \footnotesize{$12.62~(2.70)$} & \footnotesize{$0.76~(0.18)$} & \footnotesize{$0.17~(0.03)$} & \footnotesize{$0.31~(0.04)$} &  \footnotesize{$0.043~(0.04)$} \\

 \hline

\multirow{2}{*}{{ \footnotesize{BLOOM}}} & \footnotesize{OpenWebText} & 
 \footnotesize{$10.23~(1.31)$} & \footnotesize{$0.30~(0.16)$} & \footnotesize{$0.15~(0.14)$} & \footnotesize{$0.32~(0.09)$} & \footnotesize{$0.158~(0.23)$} \\
\cline{2-7}
 & \footnotesize{WikiText} & \footnotesize{$10.00~(1.47)$} & \footnotesize{$0.31~(0.13)$} & \footnotesize{$0.14~(0.13)$} & \footnotesize{$0.31~(0.09)$} & \footnotesize{$0.148~(0.23)$}  \\

 \hline

\multirow{3}{*}{{ \footnotesize{Llama 2}}} & \footnotesize{OpenWebText} & 
 \footnotesize{$\phantom{0}9.38~(1.15)$} & \footnotesize{$0.14~(0.03)$} & \footnotesize{$0.17~(0.19)$} & \footnotesize{$0.43~(0.12)$} & \footnotesize{$0.190~(0.24)$} \\
\cline{2-7}
 & \footnotesize{WikiText} & \footnotesize{$\phantom{0}8.69~(0.91)$} & \footnotesize{$0.07~(0.04)$} & \footnotesize{$0.47~(0.35)$} & \footnotesize{$0.60~(0.25)$} &  \footnotesize{$0.316~(0.27)$} \\
 \cline{2-7}
 & \footnotesize{GitHub} & \footnotesize{$\phantom{0}8.69~(1.67)$} & \footnotesize{$0.21~(0.05)$} & \footnotesize{$0.17~(0.10)$} & \footnotesize{$0.40~(0.07)$} & \footnotesize{$0.189~(0.20)$} \\
\hline
\end{tabular}
\end{center}
\vspace{-0.1in}
\end{table*}

\section{Pervasive geometrical structure}\label{sec:geometry}

We apply our decomposition to a variety of pretrained transformers including GPT-2 \cite{radford2019language}, BERT \cite{devlin2018bert}, BLOOM \cite{scao2022bloom}, Llama-2 \cite{touvron2023llama}, and various datasets such as WikiText, OpenWebText, GitHub. See Section~\ref{sec:append-setup} for details. Our geometric findings are summarized below. 

\begin{enumerate}
    \item Positional basis is a significant and approximately low-rank component, forming a continuous and curving shape. 
    \item Context basis has strong cluster patterns corresponding to documents/topics.
    \item Positional basis and context basis are nearly orthogonal (or \textit{incoherent}). 
\end{enumerate}

Our findings indicate that embeddings contain two main interpretable factors, which are decoupled due to incoherence. 


\paragraph{Sinusoidal patterns are learned, not predefined.} The transformers we examined are trained from scratch including the positional embeddings (PEs). It is unclear a priori why such a smooth and sinusoidal pattern is consistently observed. Moreover, the learned sinusoidal patterns are different from the original fixed sinusoidal embedding \cite{vaswani2017attention}: for example, 
the original PE is concentrated less on the low-frequency components (Section~\ref{sec:append-compare-PE}).

\paragraph{Consistent pattern across layers and models is not solely explained by residual connections.} The geometric structure is (i) consistent across layers and (ii) agnostic to models. 
Can consistency across layers explained by residual connections? We show this is not true:
the average norm of embeddings increases by more than 100-fold in GPT-2, and the embeddings are nearly orthogonal between layer 0 and 1 (Section~\ref{sec:append-not-residual}).

\paragraph{Some exceptions.} In the last few layers, embeddings tend to be anisotropic \cite{ethayarajh2019contextual} and geometric structure may collapse to a lower dimensional space, particularly the positional basis. It is possibly due to the completion of contextualization or optimization artifacts. We also find that BERT shows higher frequency patterns, possibly due to its different training.


\section{Key properties of positional basis: low rank and low frequency}\label{sec:pos}

\begin{figure*}[t!]
\centering
\includegraphics[width=0.8\textwidth]{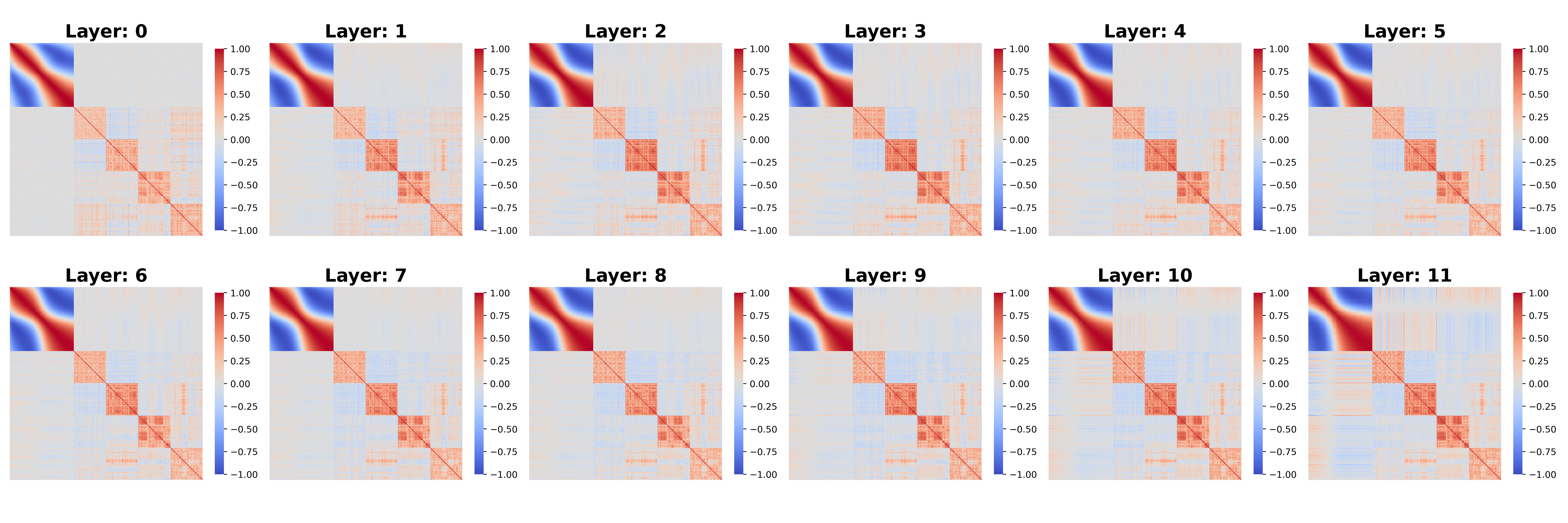}
\caption{\textbf{Normalized Gram matrix} $[\bar \mP, \bar \mC]^\top [\bar \mP, \bar \mC]$ where $\bar \mP=[\frac{\pos_1}{\norm{\pos_1}},\ldots,\frac{\pos_T}{\norm{\pos_{T}}}]$ and $\bar \mC = [\frac{\ctx_1}{\norm{\ctx_1}},\ldots,\frac{\ctx_{C}}{\norm{\ctx_C}}]$ based on GPT-2. Here, $T=128$, and $\ctx_c$ is sampled from $4$ documents with sample size $32$ in OpenWebText. We find (i) \textbf{Smoothness}, $\pos$-$\pos$ part (top left) of Gram matrix is smooth; (ii) \textbf{Incoherence}, $\pos$-$\ctx$ part (top right/bottom left) has values close to $0$; (iii) \textbf{Clustering}, $\ctx$-$\ctx$ part (bottom right) shows strong cluster patterns.}\label{fig:Gram}
\vspace{-0.1in}
\end{figure*}



In this section, we use quantitative measurements to relate smoothness to concepts in parsimonious representations.


\paragraph{Measuring smoothness.} Our notion of ``smoothness'' of the positional basis refers to its (normalized) Gram matrix:
\begin{equation}\label{def:Gram}
\mG = \bar \mP^\top \bar \mP \in \R^{T \times T}, ~ \text{where}~\bar \mP = [\frac{\pos_1}{\norm{\pos_1}},\ldots, \frac{\pos_T}{\norm{\pos_T}}]
\end{equation}
being visually smooth (mathematically, having bounded discrete derivatives). 




\subsection{Low rank via spectral analysis}

\begin{wrapfigure}[13]{r}{0.5\textwidth}
\includegraphics[width=0.45\textwidth]{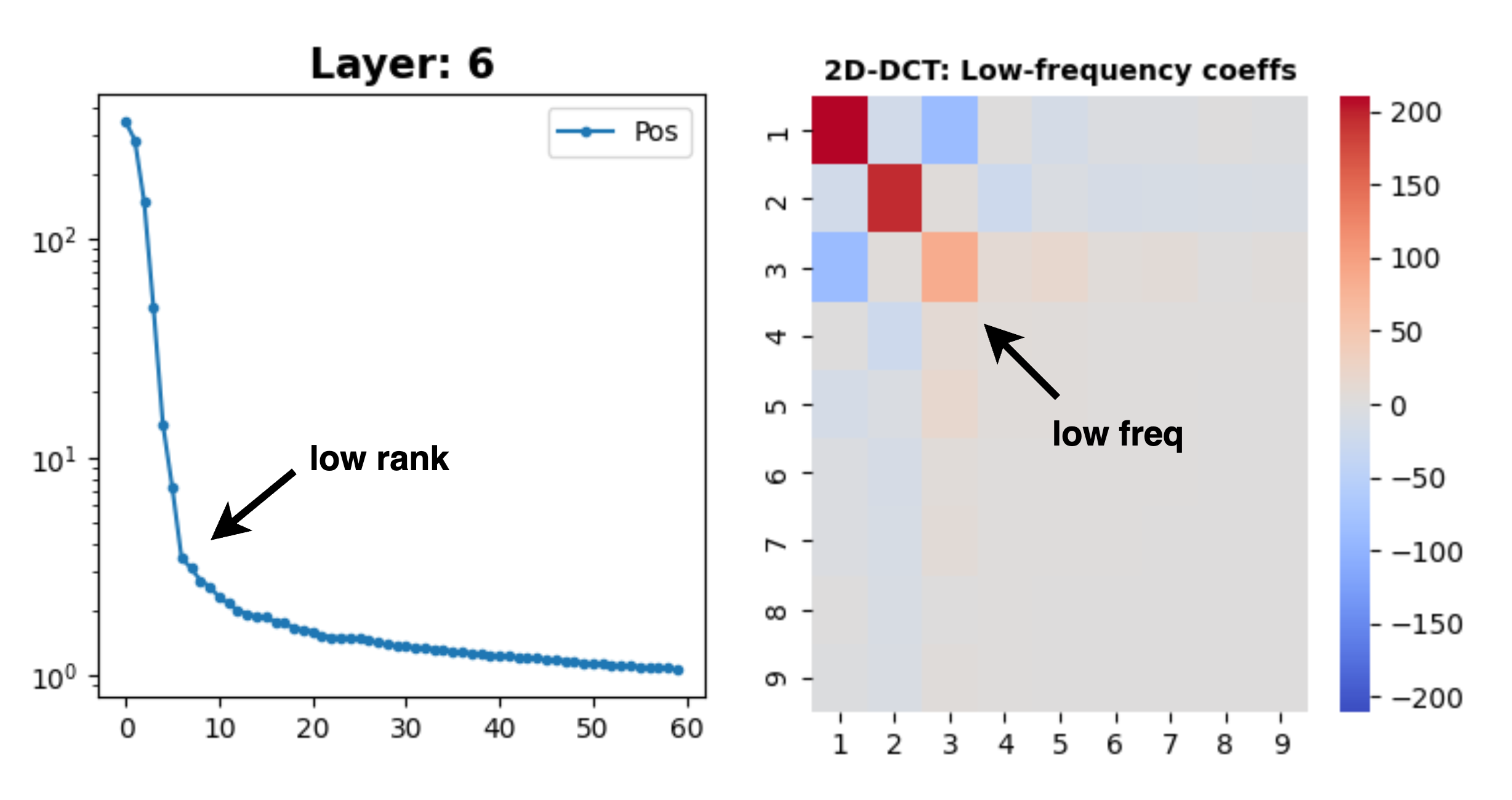}
\caption{\textbf{Spectral and Fourier analysis} based on GPT-2 model and OpenWebText. \textbf{Left:} Top-60 singular values of $\mP$. \textbf{Right:} Applying 2D discrete cosine transform to $\bar \mP^\top \bar \mP$, we show first $10$ frequency coefficients.}\label{fig:svd}
\vspace{-0.1in}
\end{wrapfigure} 


\paragraph{Low rank.} We find that the positional basis concentrates around a low-dimensional subspace. In Table~\ref{tab:main} ``rank estimate'' column, we report the rank estimate of positional basis averaged across all layers using the method of \citet{donoho2023screenot}. In Figure~\ref{fig:svd}, we plot the top singular values in descending order of $\mP=[\pos_1,\ldots, \pos_T]$. Visibly, there is a sharp change in the plot, which indicate a dominant low-rank structure. 
In Section~\ref{sec:append-lowrank}, we report detailed rank estimates.

\paragraph{Significant in relative norms.} We also find that usually, the positional basis accounts for a significant proportion of embeddings. In Table~\ref{tab:main}, we report the relative norm (averaged across layers) $\norm{\mP}_{\op} / \norm{\mM}_{\op}$, where $\mM$ contains centered embedding vectors $\vh_{c,t} - \vmu$ and columns of $\mP$ are corresponding $\pos_t$. 
The relative norms show positional basis is contributing significantly to the overall magnitude (most numbers bigger than $10\%$).


\subsection{Low frequency via Fourier analysis}\label{sec:fourier}


In Figure~\ref{fig:svd} (right), we apply the 2D discrete cosine transform to the normalized Gram matrix $\bar \mP^\top \bar \mP$; namely, we calculate the frequency matrix $\hat \mG$  by apply the (type-\rom{2}) discrete cosine transform with orthogonal matrix $\tilde \mF$:
\begin{equation*}
    \hat \mG = \tilde \mF \mG \tilde \mF^\top.
\end{equation*}
Each entry $\hat G_{ij}$ encodes the $(i,j)$-th frequency coefficient. We discover that energies are concentrated mostly in the low-frequency components, which echos the smooth and curving structure in Figure~\ref{fig:PCA}. 

\subsection{Theoretical lens: smoothness explains low-rank and low-frequency}

It is well known that the smoothness of a function is connected to fast decay or sparsity in the frequency domain \citep[Sect.~1.2.3]{pinsky2008introduction}. From the classical perspective, we establish smoothness as a critical property that induces the  observed geometry. 

\begin{center}
\textit{Smoothness of Gram matrix of positional basis induces the low-dimensional and spiral shape.}
\end{center}

For convenience, we assume that positional vectors $\pos_t$ have unit norm, so by definition,
$\pos_1+\ldots+\pos_T=\vzero $. To quantify smoothness, we introduce the definition of finite difference. As with the discrete cosine transform in 1D, we extend and reflect the Gram matrix to avoid boundary effects.

Let $\mG^{(1)} = \mG$ and $\mG^{(2)}, \mG^{(3)}, \mG^{(4)} \in \R^{T \times T}$ be defined by $\mG^{(2)}_{t,t'} = \mG_{t,T+1-t'}$, $\mG^{(3)}_{t,t'} = \mG_{T+1-t, t'}$, $\mG^{(4)}_{t,t'} = \mG_{T+1-t, T+1-t'}$ for any $t,t'=1,2,\ldots T$. We extend and reflect $\mG$ by
\begin{equation}\label{def:tildeG}
\tilde \mG := \left(\begin{array}{cc} \mG^{(1)} & \mG^{(2)} \\ \mG^{(3)} & \mG^{(4)}\end{array}\right)\; .
\end{equation}
We define the first-order finite difference by (using periodic extension $\tilde G_{t\pm2T, t'\pm 2T} = \tilde G_{t,t'}$)
\begin{equation}
[\Delta^{(1,1)} \tilde \mG]_{t,t'} = T^2\big( \tilde G_{t,t'} - \tilde G_{t-1,t'} - \tilde G_{t,t'-1} + \tilde G_{t-1,t'-1} \big)
\end{equation}
for all integers $t,t'$. Higher-order finite differences are defined recursively by $\Delta^{(m,m)} \tilde \mG = \Delta^{(1,1)} \big( \Delta^{(m-1,m-1)} \tilde \mG \big)$.

Note that $\Delta^{(m,m)} \tilde \mG$ measures higher-order smoothness of $\tilde \mG$. Indeed, if $G_{t,t'} = f(t/T,t'/T)$ for certain smooth function $f(x,y)$ defined on $[0,1]^2$, then $[\Delta^{(m,m)} \tilde \mG]_{t,t'} \approx \partial_x^m \partial_y^m f(t/T,t'/T)$.

\begin{thm}\label{thm:fourier}
   Fix positive integers $k \le T$ and $m$. Define the low-frequency vector $\vf_s = (1, \cos((s-0.5)\pi/T), \ldots, \cos ((s-0.5)(T-1)\pi / T))^\top \in \R^T$ where $s=1,\ldots,k$, and denote $\mF_{\le k} = [\vf_1,\ldots,\vf_k] \in \R^{T \times k}$. Then there exists $\mB \in \R^{k \times k}$ such that
   \begin{equation*} \label{ineq:fourier}
   \frac{1}{T} \left \lVert \mG - \mF_{\le k} \mB (\mF_{\le k} \mB)^\top \right \rVert_\op \le \frac{6}{(8k)^m} \norm{\Delta^{(m,m)} \tilde \mG}_{\max}\, .
   \end{equation*}
\end{thm}
This theorem implies that if the extended Gram matrix has higher-order smoothness, namely $\norm{\Delta^{(m,m)} \tilde \mG}_{\max}$ is bounded by a constant, then even for moderate $k$ and $m$, we have approximation $\mG \approx \mF_{\le k} \mB (\mF_{\le k} \mB)^\top $. Note that $\mF_{\le k} \mB$ consists of linear combinations of low-frequency vectors. This explains why $\mG$ has a dominant low-rank and low-frequency component.

\section{Smoothness: blessing of natural languages}\label{sec:smoothness}

We demonstrate that smoothness is a natural and beneficial property learned from language data: it is robust to distribution shifts and allows efficient attention computation. This smoothness likely reflects the nature of language data.

\subsection{Smoothness is robust in out-of-distribution data}\label{sec:pos-ood}

So far, we have analyzed our decomposition and associated geometry primarily on \textit{in-distribution} samples. For example, we sample sequences from OpenWebText---the same corpus GPT-2 is pretrained on, to calculate decomposition \eqref{def:components}--\eqref{def:components2}. 

Now we sample out-of-distribution (OOD) data and conduct similar decomposition and analyses. We find that the positional basis possesses similar low-rank and spiral structure. Surprisingly, for sequences consisting of randomly sampled tokens, such structure persists. See summaries in Table~\ref{tab:ood} and details in Section~\ref{sec:append-pos-ood}.

Many LLMs can generalize well to OOD data. We believe that this generalization ability is at least partially attributed to the robustness of positional basis as evidenced here.

\begin{table}
\caption{\textbf{Robustness of positional basis}. Similar geometric structures found on \textit{OOD} samples: NanoGPT (trained on a Shakespeare dataset) evaluated on WikiText, GPT-2 and BERT (trained on language datasets) evaluated on GitHub data, BLOOM evaluated on random tokens.}\label{tab:ood}
\centering
\begin{tabular}{lrr}
\toprule
  & rank estimate of $\mP$ & ratio explained by low-freq ($K=10$)  \\ \midrule
\footnotesize{NanoGPT}  & \footnotesize{$\phantom{0}5.43~(1.84)$} & \footnotesize{$84.0\%$}    \\
 \footnotesize{GPT-2}  & 
 \footnotesize{$11.54~(1.55)$} & \footnotesize{$99.5\%$}   \\
\footnotesize{BERT}  & \footnotesize{$12.46~(2.47)$} & \footnotesize{$70.2\%$}   \\
 \footnotesize{BLOOM}  & \footnotesize{
 $10.44~(2.25)$
 } & \footnotesize{$95.6\%$} \\ \bottomrule
\end{tabular}
\end{table}

\subsection{Smoothness promotes local and sparse attention}\label{sec:local-attn}

\begin{table}
\caption{\textbf{Inheriting smoothness from positional basis}. For $K=1,3,5,10$, we apply 2D DCT and calculate the ratio\protect\footnotemark of up to $K$-th low-frequency coefficients based on GPT-2 and positional basis.}\label{tab:local-attn}
\centering
\begin{tabular}{lrrrr}
\toprule
Ratio explained & $K=1$ & $K=3$ & $K=5$ & $K=10$ \\ \midrule
$\mP \mP^\top$ & 39.4\% & 95.4\% & 98.2\% & 99.7\% \\ 
$\mP \mW \mP^\top$ & 45.9\% & 96.5\% & 98.8\% & 99.9\% \\ 
QK matrix & 84.0\% & 90.0\% & 90.4\% & 91.2\% \\ \bottomrule
\end{tabular}
\end{table}

Many attention heads in large-scale transformers for language data show local and sparse patterns \citep{beltagy2020longformer}. A usual heuristic explanation is that, most information for token prediction is contained in a local window.

We offer an alternative explanation: \textit{Local attention is a consequence of smoothness of positional basis}.

We provide our reasoning. First, in terms of smoothness, the Gram matrix $\mP \mP^\top$ is closely related to $\mP \mW \mP^\top$, where $\mW = \mW^q (\mW^k)^\top / \sqrt{d_{\head}}$. Indeed, generally a linear transformation of the positional vectors, namely $\mW \mP^\top$, should not affect smoothness much.

Second\footnotetext{Given 2D frequencies $(f_{ij})$, we calculate the ratio as $r_K = \sum_{i,j \le K} f_{ij}^2 / \sum_{i,j} f_{i,j}^2$.}, $\mP \mW \mP^\top$ is an important---often dominant---constituent of the QK matrix (QK matrix is the matrix inside softmax in \eqref{AttnHead}); see Section~\ref{sec:beyond-attn} for an example. Thus, the QK matrix can inherit the smoothness from $\mP \mW \mP^\top$.


Third, if the QK matrix (denoted by $\mB$) is smooth and maximized at the diagonal position along a certain row indexed by $t$ (the constraint $t' \le t$ is due to causal masking):
\begin{equation}\label{eq:argmax}
    \argmax_{1 \le t' \le t} B_{t, t'} = t,
\end{equation}
then positions $t'$ close to $t$ also have high QK values by smoothness. Therefore, we expect the neighboring positions around $t$ to receive high attention weights after softmax.


Our reasoning is supported by pervasive smoothness shown in Table~\ref{tab:local-attn}. Moreover, we find that the positional constituents in more than 43\% heads in GPT-2 satisfy \eqref{eq:argmax} for more than 80\% positions $t$. See Section~\ref{sec:append-local-attn} for details.

\begin{figure*}[t!]
\centering
\includegraphics[width=0.9\textwidth]{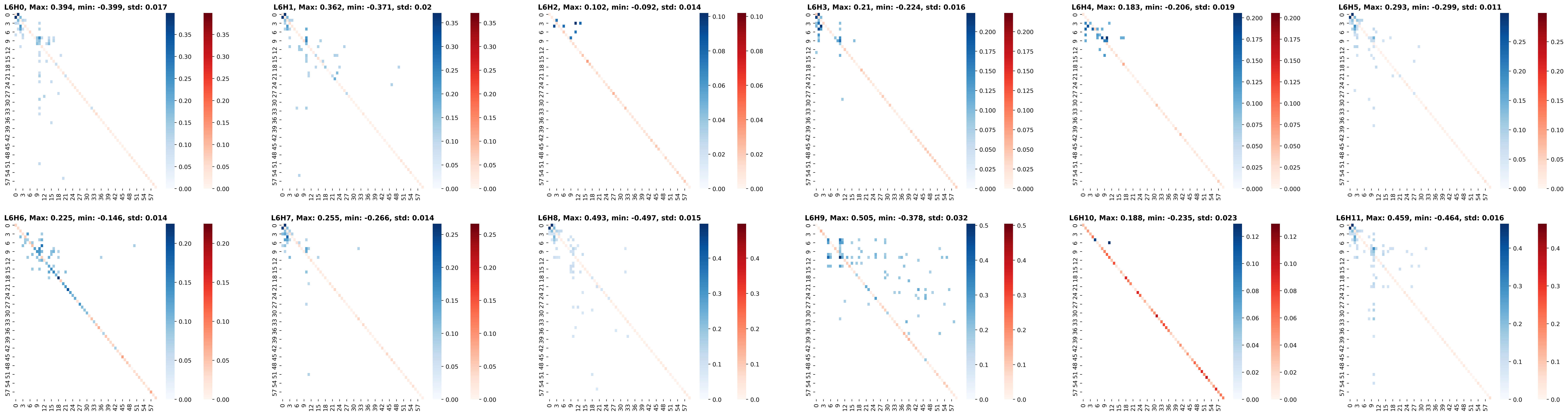}
\caption{\textbf{Decoupling trained weight matrices.} For $12$ attention heads (layer $L=6$ shown here) in GPT-2, we study the matrix $\mW = \mW^q (\mW^k)^\top / \sqrt{d_{\head}} \in \R^{d \times d} $. {\color{red}Red}: diagonal entries ${\color{red}\mD}:=\diagg(\mW)$. {\color{blue}Blue}: take off-diagonal matrix $\mW - \diagg(\mW)$, rotate it by the right singular vectors of positional basis $\mV$, then apply denoising.
Large absolute values concentrate in small top-left part ${\color{blue}\mL}$.
}\label{fig:attention-weight}
\vspace{-0.1in}
\end{figure*}

\subsection{Curse of discontinuity for arithmetic tasks}\label{sec:addition}

\begin{wrapfigure}[15]{r}{0.5\textwidth}
\includegraphics[width=0.45\textwidth]{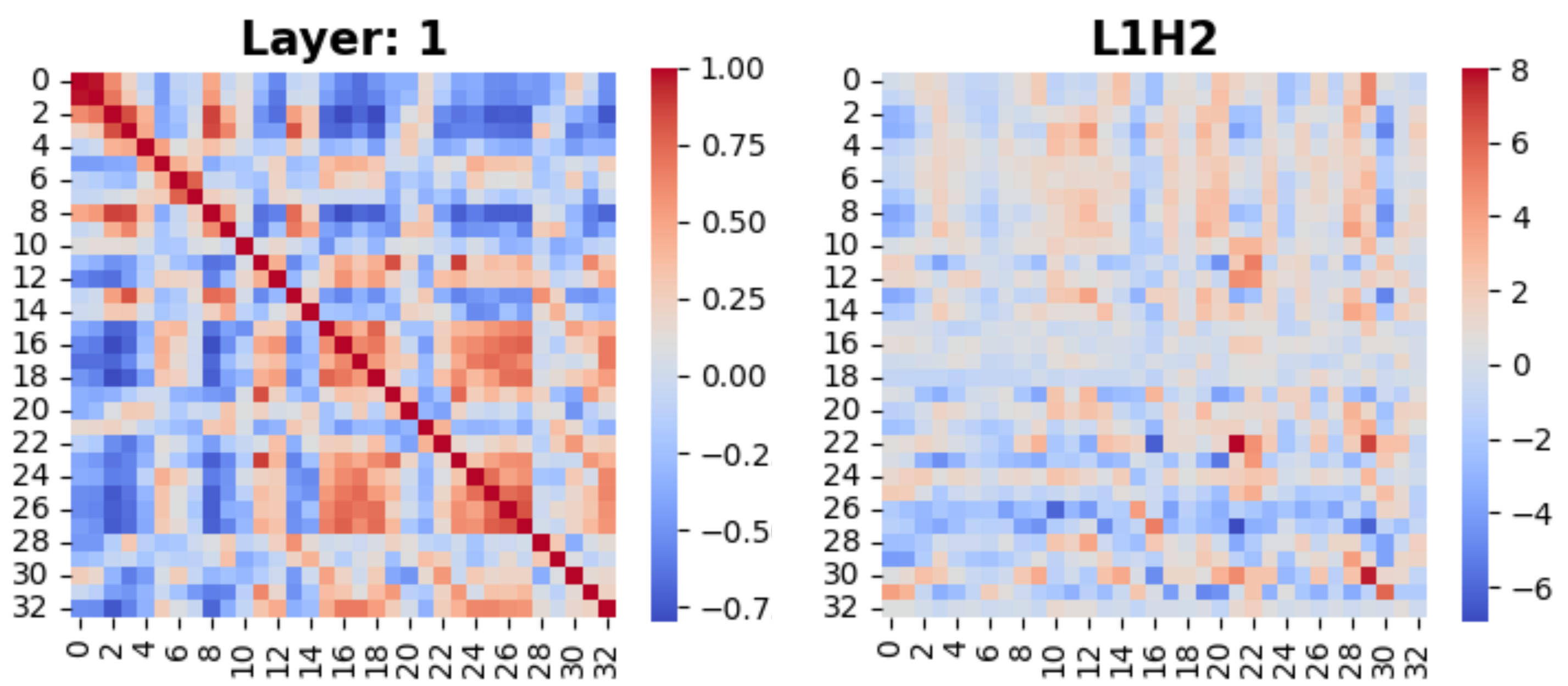}
\caption{Addition task trained on NanoGPT exhibits \textbf{nonsmooth} patterns: \textbf{discontinuity} as a consequence of non-language data training. \textbf{Left}: Gram matrix of normalized positional basis. Compare with top-left of plots in Figure~\ref{fig:Gram}. \textbf{Right}: QK matrix. 
}\label{fig:discontinuity}
\end{wrapfigure}

We find that the emergence of smoothness is data dependent: while transformers pretrained on natural/programming languages exhibit the smoothness property, they may suffer from a lack of smoothness on other pretrained data.

We explore a simple arithmetic task---\textit{Addition}, where inputs are formatted as a string ``$a + b = c$'' with $a,b,c$ represented by digits of a certain length. We sample the length of each addition component uniformly from $\{L/2, \ldots, L\}$ where $L=10$. Then, we train a 4-layer 4-head transformer (NanoGPT) with character-level tokenization to predict `$c$' based on the prompt ``$a+b=$''. We train this transformer 10,000 iterations until convergence. 

\paragraph{Nonsmooth pattern.} Figure~\ref{fig:discontinuity} shows that the Gram matrix of normalized positional basis and QK matrix are visibly discontinuous and exhibit many fractured regions. Quantitatively, the top-10 low-frequency components of Gram matrix explains around 50\% of $\sum_{ij}\hat G_{ij}^2$ (Section~\ref{sec:append-addition}), much less than 99\% in GPT-2 as in Table~\ref{tab:local-attn}. This suggests a sharp distinction between language vs.~non-language data in terms of induced geometry.

\paragraph{Failure of length generalization.} Our NanoGPT achieves above 99\% in-distribution test accuracy, yet fails at OOD generalization: it has less than 20\% accuracy on average on digits of length smaller than 5. We believe nonsmoothness is likely an intrinsic bottleneck for arithmetic tasks.

Our results hold for transformers with relative positional embeddings as well; see Section~\ref{sec:append-addition}.


\section{Incoherence enhances interpretability}\label{sec:incoh}

Near-orthogonality, or (mutual) incoherence, is known to be a critical property for sparse learning. Generally speaking, incoherence means that factors or features are nearly uncorrelated and decoupled. In the ideal scenario of orthogonality, factors can be decomposed into orthogonal non-intervening subspaces. Incoherence is closely related to \textit{restricted isometry} \citep{candes2005decoding}, \textit{irrepresentable conditions} \citep{zhao2006model}, etc.  

We observe the incoherence property between the positional basis and the context basis, as Table~\ref{tab:main} shows that $\max_{t,c} |\langle \frac{\pos_t}{\norm{\pos_t}}, \frac{\ctx_c}{\norm{\ctx_c}} \rangle|$ is typically small. The low incoherence in Table~\ref{tab:main} (random baseline is about 0.1) means that the two bases are nearly orthogonal to each other.

To decouple different effects,
in this section, we will focus on positional effects vs.~non-positional effects, thus working with $\cvec_{c,t}$ instead of $\ctx_c$.

\subsection{Decoupling positional effects in pretrained weights}\label{sec:weights-decomp}

We present evidence that the trained weight matrix $\mW:= \mW^q (\mW^k)^\top / \sqrt{d_{\head}}$ in self-attention has a clear and interpretable decomposition: heuristically, a low rank component that modulates the positional effects, and a diagonal component that strengthens or attenuates token effects.

More precisely, we identify a common \textit{low-rank plus noise} structure in the majority of attention heads in pretrained transformers.
\begin{equation}\label{eq:weight-matrix}
\mW = \underbrace{\mV \mL \mV^\top}_{\text{low rank}} + \underbrace{\mD}_{\text{diagonal}} + \; \textrm{Noise} \; .
\end{equation}
Here, the columns of $\mV \in \R^{d \times K}$ are the top-$K$ right singular vectors of positional basis matrix $\mP$, and $\mL \in \R^{K \times K}$, where $K$ is not large.

Figure~\ref{fig:attention-weight} shows empirical support to structural claim \eqref{eq:weight-matrix}. Given a pretrained weight $\mW$, we take $\mD = \diagg(\mW)$ and show entries in red. Then, we rotate the off-diagonal part of $\mW$ by the right singular vectors of $\mP$ and apply denoising, namely zeroing entries whose absolute values are smaller than a threshold. For many heads, the surviving large absolute values are concentrated in the top left ($K\approx 20$)---which suggests that indeed a significant component of $\mW$ is aligned with the positional basis.

This decomposition suggests a possible mechanism inside attention computation. Consider an ideal scenario where $\mD$ is a multiple of identity matrix, and each embedding has orthogonal decomposition $\vh = \vt + \vc$ with $\vt \in \spann(\mV)$ encoding positional information and $\vc \in \spann(\mV)^\bot$ encoding non-positional information. Then, for two embedding vectors $\vh, \vh'$ with $\vh=\vt + \vc$ and $\vh'=\vt' + \vc'$,
\begin{equation*}
\vh^\top \mW \vh' \approx \underbrace{\vt^\top (\mV \mL \mV^\top + \mD ) \vt'}_{\text{positional effect}} \;+ \underbrace{\vc^\top \mD \vc'}_{\text{context / token effect}} 
\end{equation*}
Positional effects are decoupled from context effects, since cross terms involving $\vt, \vc'$ or $\vt', \vc$ vanish. This heuristics may allow us to examine how information is processed in attention heads, which will be explored in future work. 

\subsection{Beyond attention visualization}\label{sec:beyond-attn}

Attention visualization is often used as a diagnosis of self-attention in transformers. We show that more structure can be revealed by decomposing the QK matrix 
using our decomposition of embeddings.

\begin{figure*}[t!]
\centering
\includegraphics[width=0.95\textwidth]{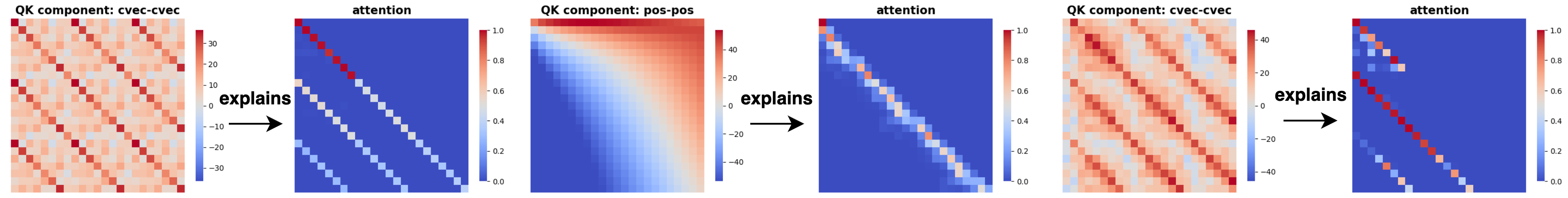}
\caption{\textbf{QK constituent plots enhance attention visualization}: $3$ representative attention patterns in induction heads can be explained by their respective dominant QK constituents. \textbf{Left pair:} $\cvec$-$\cvec$ constituent drives attention to identical tokens. \textbf{Middle pair:} $\pos$-$\pos$ constituent determines attention to neighboring tokens. \textbf{Right pair:} $\cvec$-$\cvec$ shows attentions to shifted tokens.
}\label{fig:induction}
\end{figure*}

We start with decomposing the QK matrix. Assuming $\vmu = \vzero$ (ignoring global mean effect for convenience), then for embedding vectors $\vh, \vh' \in \R^d$ we have
\begin{align}
    & \vh^\top \mW^q (\mW^k)^\top \vh = \pos^\top \mW^q (\mW^k)^\top \pos \notag\\
    & + \pos^\top \mW^q (\mW^k)^\top \cvec 
    + \cvec^\top \mW^q (\mW^k)^\top \pos \notag\\
    &+ \cvec^\top \mW^q (\mW^k)^\top \cvec\,.\label{QKdecomp}
\end{align}
Each of the four components shows how much an attention head captures information from cross-pairs $\pos/\cvec$---$\pos/\cvec$ of an embedding. We propose to visualize each QK component separately, which tells if an attention head captures positional information or context/token information from embeddings.

\paragraph{Case study: induction heads.} \citet{elhage2021mathematical} identified components in transformers that complete a sequence pattern based on observed past tokens, namely, predicting the next token $[B]$ based on observed sequence $[A], [B], \ldots, [A]$. This ability of copying previous tokens in known to be caused by \textit{induction heads}.

There are three representative attention patterns as shown in Figure~\ref{fig:induction}: (i) attention to identical tokens, namely attention weights concentrate on tokens identical to the present token, (ii) attention to neighboring tokens, (iii) attention to tokens to be copied. See also \citet{elhage2021mathematical}.


Attention visualization alone does not reveal why certain heads emerge. In contrast, our QK decomposition reveals which QK constituent is dominant and responsible for attention patterns. For example, in Figure~\ref{fig:induction},
\begin{itemize}
    \item attention to neighboring tokens (middle plots) is predominately determined by the $\pos$-$\pos$ constituent;
    \item attention to identical token (left) or shifted token (right) is determined by the $\cvec$-$\cvec$ constituent. This implies that induction heads are not based on memorizing relative positions, but on matching token information.
\end{itemize}
More details are in Section~\ref{sec:append-QK}.

\subsection{Theoretical insight from kernel factorization}

Why does incoherence emerge from pretrained transformers? While it is well known in sparse coding and compressed sensing that \textit{handcrafted} incoherent basis facilitates recovery of sparse signals \citep{donoho1989uncertainty, donoho2003optimally, donoho2006compressed, candes2006robust}, it is surprising that incoherence arises from \textit{automatic feature learning}.

Here we focus on the self-attention mechanism of transformers. By adopting the kernel perspective, we present a preliminary theory for the 
following heuristics:
\begin{center}
\textit{
Incoherence enables a kernel to factorize into smaller components, each operating independently.}
\end{center}
Given query/key matrices $\mW^q, \mW^k \in \R^{d \times d_{\head}}$, we define the (asymmetric) kernel by (recalling $\mW = \mW^q (\mW^k)^\top / \sqrt{d_{\head}}$)
\begin{equation*}
    K_{\mW}(\vz, \vz') := \exp \left( \vz^\top \mW \vz' \right) = \exp \left( \frac{\langle \mW^q \vz, \mW^k \vz' \rangle}{\sqrt{d_{\head}}} \right).
\end{equation*}
Using $K_{\mW}$, the attention can be expressed as kernel smoothing: for embeddings  $(\vx_t)_{t\le T} \subset \R^d$,
\begin{equation}\label{def:Attn}
    \AttnHead(\vx_t; K_{\mW}) = \sum_{k \le t} \frac{K_{\mW}(\vx_{k}, \vx_t)}{\sum_{k' \le t} K_{\mW}(\vx_{k'}, \vx_t)} v(\vx_k)
\end{equation}
where $v: \R^d \to \R$ is a generic value function. This kernel perspective is explored in \citet{tsai2019transformer}, where it is argued that the efficacy of self-attention largely depends on the form of the kernel.

Suppose that there are two overcomplete bases $\gB_1^0, \gB_2^0 \subset \R^d$. For simplicity, assume that $\norm{\vu}_2 \le 1$ if $\vu \in \gB_1^0$ or $\gB_2^0$. 
The mutual incoherence is $\incoh := \max\big\{ | \langle \vc, \vt \rangle: \vc\in \gB_1^0, \vt \in \gB_2^0 \big\}$. Consider the (extended) overcomplete basis $\gB_\alpha := \{\lambda \vu: \vu \in \gB_\alpha^0, \lambda \in [-1,1]\}$ where $\alpha \in \{1,2\}$. Given query/key vectors $\vx^q, \vx^k \in \R^d$, suppose that we can decompose them according to the two bases.
\begin{equation}\label{eq:decompbasis}
    \vx^q = \vc^q + \vt^q, \quad \vx^k = \vc^k + \vt^k, ~ \text{for}~\vc^q,\vc^k \in \gB_1;~ \vt^q,\vt^k \in \gB_2.
\end{equation}

Generically, we can decompose $K_{\mW}(\vx^q, \vx^k)$ into
\begin{equation*}
    K_{\mW}(\vc^q, \vc^k) K_{\mW}(\vc^q, \vt^k) \cdot K_{\mW}(\vt^q, \vc^k) K_{\mW}(\vt^q, \vt^k)\, .
\end{equation*}


Each component measures cross similarity of pairs between $\vc^q,\vt^q$ and $\vc^k,\vt^k$, which then translates into a weight for the attention. Unfortunately, this general decomposition requires the individual kernels to share the same weight $\mW$, which hinders capturing cross interactions flexibly. 

It turns out that if the weight matrix is \textit{sparsely represented} by the bases, then kernel flexibility can be achieved. To be precise, we will say that $\mW \in \R^{d \times d}$ is $s$-sparsely represented by bases $\gB, \gB'$ if there exist $(a_k)_{k\le s} \subset [-1,1]$, $(\vu_k)_{k\le s} \subset \gB$, $(\vv_k)_{k\le s} \subset \gB'$ such that
\begin{equation}\label{eq:mW}
    \mW = \sum_{k\le s} a_k \vu_k \vv_k^\top.
\end{equation}

\begin{thm}\label{thm:incoh}
    Let $\mW_{11}, \mW_{12}, \mW_{21}, \mW_{22} \in \R^{d \times d}$ be any matrices with the following properties: for $\alpha, \beta \in \{1,2\}$, $\mW_{\alpha \beta} \in \R^{d \times d}$ is $O(1)$-sparsely represented by bases $\gB_\alpha, \gB_\beta$. Then for all $\vx^q,\vx^k \in \R^d$ satisfying \eqref{eq:decompbasis}, $\mW = \mW_{11} + \mW_{12} + \mW_{21} + \mW_{22}$ satisfies
    \begin{align}
        K_{\mW}(\vx^q, \vx^k) = &\big(1 + O(\incoh) \big) \cdot K_{\mW_{11}}(\vc^q, \vc^k) K_{\mW_{12}}(\vc^q, \vt^k) \notag \\ & \cdot 
 K_{\mW_{21}}(\vt^q, \vc^k) K_{\mW_{22}}(\vt^q, \vt^k) \label{eq:incoh}
    \end{align}
    Moreover, \eqref{eq:incoh} holds with probability at least $1-O((|\gB_1^0|\cdot|\gB_2^0|)\exp(-\incoh^2\cdot d)$ if each $\mW_{\alpha \beta}$ is replaced by $\mW_{\alpha \beta} + \frac{\rmZ_{\alpha \beta}}{\sqrt{d}}$ where $(\rmZ_{\alpha \beta})_{kk'}$ is an independent subgaussian\footnote{We say that a random variable $\rxi$ is subgaussian if $\E[\rxi]=0$ and $\E[\exp(\lambda \rxi)] \le \exp(\lambda^2/2)$ for all $\lambda \in \R $.} random variable.
\end{thm}
The factorization \eqref{eq:incoh} says that each kernel component has a separate weight matrix, and all components contribute multiplicatively to $K_\mW$. The ``moreover'' part generalizes the sparse representation notion by allowing additive noise, which matches the empirical structure in \eqref{eq:weight-matrix}. The additive construction of $\mW$ is connected to \textit{task arithmetic} \citep{ilharco2022editing, ortiz2023task} recently studied.

\begin{rmk}
If we suppose $\incoh \asymp d^{-\gamma}$ with $1/2 > \gamma > 0$, then the high probability statement is nontrivial if $|\gB_1^0| \cdot |\gB_2^0| = o(\exp(d^{1-2\gamma}))$. This dictionary size limit is generally reasonable.
\end{rmk}

\section{Related work}\label{sec:related}

Analyses of transformers have attracted research interest since  \citet{vaswani2017attention}. Many studies on GPT-2 \citep{radford2019language} and BERT \citep{devlin2018bert} show that last-layer contextualized embeddings capture linguistic structure and exhibit excellent downstream performance \citep{hewitt2019structural, chi2020finding, thompson2020topic}. Fewer papers focus on the geometry or intermediate-layer embeddings: in \citet{ethayarajh2019contextual}, it is found that later-layer embeddings are increasingly anisotropic and context-specific; \citet{cai2020isotropy,reif2019visualizing,hernandez2021low,gao2019representation} observed interesting geometric structures and artifacts without thorough analysis; \citet{yeh2023attentionviz} provide visualization tools for embeddings. Very recent papers provide empirical/theoretical evidence about either low-rank or diagonal structure in attention weight matrices \citep{boix2023transformers,trockman2023mimetic}. Our decomposition 
unifies scattered empirical phenomena, reveals consistent geometry and explains observed artifacts (anisotropic, spiral shape, etc.).

Many variants of positional embedding are proposed \citep{shaw2018self, dai2019transformer, su2021roformer, scao2022bloom, press2021train} since \citet{vaswani2017attention}. Since GPT-4, many papers focus on length generalization for arithmetic tasks \citep{kazemnejad2023impact, lee2023teaching}. 
Prior analyses on positional embeddings focus only on static (0-th layer) embeddings for selected transformers \citep{wang2020onposition, ke2020rethinking, wang2020position, tsai2019transformer,yamamoto2023absolute},
whereas we provide a more detailed picture.

Prior work on LSTMs finds decomposition-based methods enhance interpretability \citep{murdoch2018beyond}. Understanding the inner workings of transformers is usually done through attention visualization \citep{clark2019does, wang2022interpretability}. The emergence of induction heads \citep{elhage2021mathematical, inductionhead22} 
is supported by attention visualization,
which is further reinforced by our analysis.

\section{Limitations}\label{sec:limit}

In this paper, we 
mostly focus
on pretrained transformers due to limited computational resources. It would be interesting to investigate the impact of input/prompt formats on the geometry of embeddings over the course of training, especially for different linguistic tasks and arithmetic tasks.

Also, we mostly focus on the mean vectors $\pos_t$ and $\ctx_c$ but not study $\resid_{c,t}$ thoroughly. We find that the residual component is not negligible (e.g., containing token-specific information). It would be interesting to study the higher-order interaction in $\resid_{c,t}$ and propose a nonlinear decomposition of embeddings, which is left to future work.

\section{Acknowledgement}\label{sec:ack}
We thank Junjie Hu, Tim Ossowski, Harmon Bhasin, Wei Wang for helpful discussions. 

Support for this research was provided by the Office of the Vice Chancellor for Research and Graduate Education at the University of Wisconsin--Madison with funding from the Wisconsin Alumni Research Foundation.

\newpage

\bibliography{refs}
\bibliographystyle{icml2024}

\newpage
\appendix
\addcontentsline{toc}{section}{Appendix} 
\part{Appendix} 
\parttoc 

\input{Appendix/SectionA}
\input{Appendix/SectionB}

\input{Appendix/SectionC}
\input{Appendix/SectionD}
\input{Appendix/SectionE}

\input{Appendix/SectionF}

\end{document}

%% file: math_commands.tex
\usepackage{amsmath,amssymb,amsfonts,amsthm, bm}

\def\1{\bm{1}}

\def\rxi{{\textnormal{$\xi$}}}

\def\rmZ{{\mathbf{Z}}}

\def\vzero{{\bm{0}}}
\def\vone{{\bm{1}}}
\def\vmu{{\bm{\mu}}}

\def\vb{{\bm{b}}}
\def\vc{{\bm{c}}}

\def\vf{{\bm{f}}}

\def\vh{{\bm{h}}}

\def\vt{{\bm{t}}}
\def\vu{{\bm{u}}}
\def\vv{{\bm{v}}}

\def\vx{{\bm{x}}}

\def\vz{{\bm{z}}}

\def\mA{{\bm{A}}}
\def\mB{{\bm{B}}}
\def\mC{{\bm{C}}}
\def\mD{{\bm{D}}}

\def\mF{{\bm{F}}}
\def\mG{{\bm{G}}}

\def\mL{{\bm{L}}}
\def\mM{{\bm{M}}}

\def\mP{{\bm{P}}}
\def\mQ{{\bm{Q}}}

\def\mV{{\bm{V}}}
\def\mW{{\bm{W}}}
\def\mX{{\bm{X}}}

\def\mGamma{{\bm{\Gamma}}}

\DeclareMathAlphabet{\mathsfit}{\encodingdefault}{\sfdefault}{m}{sl}
\SetMathAlphabet{\mathsfit}{bold}{\encodingdefault}{\sfdefault}{bx}{n}

\def\gB{{\mathcal{B}}}

\def\gI{{\mathcal{I}}}
\def\gJ{{\mathcal{J}}}

\newcommand{\E}{\mathbb{E}}

\newcommand{\R}{\mathbb{R}}
\newcommand{\C}{\mathbb{C}}

\newcommand{\softmax}{\mathrm{softmax}}

\newcommand{\diag}{\mathrm{diag}}
\newcommand{\op}{\mathrm{op}}

\DeclareMathOperator*{\argmax}{arg\,max}

\providecommand{\norm}[1]{\lVert#1\rVert}

\newtheorem{thm}{Theorem}

\newtheorem{lem}{Lemma}
\newtheorem{rmk}{Remark}

%% file: Appendix/SectionA.tex
\section{Models, datasets, and implementations}\label{sec:append-setup}

We present the details of our experiments and measurements.

\subsection{Pretrained models}

We downloaded and used pretrained models from Huggingface. In addition, we trained a  \href{https://github.com/karpathy/nanoGPT} {NanoGPT} on a \href{https://raw.githubusercontent.com/karpathy/char-rnn/master/data/tinyshakespeare/input.txt}{Shakespeare} dataset and on addition tasks.

\begin{itemize}
    \item GPT-2 \citep{radford2019language}: 12-layer, 12-head, 768-dim, 124M parameters, autoregressive, absolute positional encoding at $0$th-layer, pretrained on OpenWebText;
    \item BERT \citep{devlin2018bert}: 12-layer, 12-head, 768-dim, 124M parameters, masked prediction, absolute positional encoding at $0$th-layer, pretrained on BooksCorpus and English Wikipedia;
    \item BLOOM \citep{scao2022bloom}: 24-layer, 16-head, 1024-dim, 560M parameters, ALiBI positional encodings \citep{press2021train} at each layer, pretrained on 45 natural languages and 12 programming languages;
    \item Llama2-7B \citep{touvron2023llama}: 32-layer, 32-head, 4096-dim, 7B parameters, autoregressive, Rotary positional embedding \citep{su2021roformer} at every layer, pretrained on a variety of data.
\end{itemize}

Note that (i) the training objective for pretraining BERT is different from the other models, and (ii) Llama 2 uses rotary positional encoding for each layer and BLOOM uses ALiBI positional encoding---which is different from absolute positional encoding that is added at the $0$-th layer \citep{vaswani2017attention}.

\subsection{Training small transformers}

We train a few smaller transformers in this paper. Models are based on the GPT-2 architecture with adjusted parameters, and we adopt the implementation of the \href{https://github.com/karpathy/nanoGPT}{GitHub Project} by Andrej Karpathy. The hardware we use is mainly RTX3090ti. All the following experiments take under 1 hour to train.

\begin{itemize}
    \item \textbf{NanoGPT in Table~\ref{tab:main} and~\ref{tab:ood}}: The model is a vanilla Transformer with 6 layers, 6 heads, 384 dimensional embeddings, residual/embedding/attention dropout set to 0.1, weight decay set to 0.1, and a context window of 128. The dataset is Shakespeare with character-level tokenization. We train 5000 iterations using the AdamW optimizer, with a  batch size of 64 and a cosine scheduler (100 step warmup) up to a learning rate of 1e-3;
    
    \item \textbf{Addition}: Similarly, we use a vanilla Transformer with 4 layers, 4 heads, 256 dimensional embeddings, and weight decay set to 0.1. The context window is set as the length of the longest sequence, i.e., 33 for the 10-digits addition task here. We train 10,000 iterations using the AdamW optimizer, with a  batch size 64 and a linear scheduler up to a learning rate of 1e-4.
\end{itemize}


\subsection{Removing artifacts}

There are two likely artifacts in the measurements and visualization that we removed in the paper. 
\begin{enumerate}
    \item First token in a sequence. We find that a large proportion of attention is focused on the first token, which usually distorts visualization significantly. It has been known that the first token functions as a ``null token'', which is removed in analysis \citep{vig2019analyzing}. We also adopt removing the first token in our measurements and visualization.
    \item Final-layer embeddings. We find that the embeddings of the final layer typically do not have a significant positional basis component. It is likely that positional information is no longer needed since last-layer embeddings are directly connected to the loss function.
\end{enumerate}

\subsection{Positional basis calculation}

We calculate positional bases based on sampled sequences of length $T$ from a subset of the corpus, which includes OpenWebText, WikiText, and GitHub. The implementation and weights of the pretrained models are obtained from HuggingFace.

For the curated corpus subset, we utilize the streaming version of the HuggingFace datasets and extract the first 10K samples from the train split. Then we tokenize the dataset using the same tokenizer employed by the pretrained model. The size of the final datasets vary across tasks and datasets, and we ensure that there are at least 1M tokens in each case to prevent the occurrence of overlapping sequences.

We set the context window $T$ = 512 for BERT, BLOOM, and GPT-2, as this maintains the maximum context window utilized during pretraining. For Llama-2, we set $T$ = 512 instead of a longer maximum sequence length due to computational resource limitations.

%% file: Appendix/SectionB.tex
\section{Additional empirical results for Section~\ref{sec:geometry}}

\subsection{Comparison with original sinusoidal positional embedding}\label{sec:append-compare-PE}

\begin{figure}[h]
\centering
\includegraphics[width=0.95\textwidth]{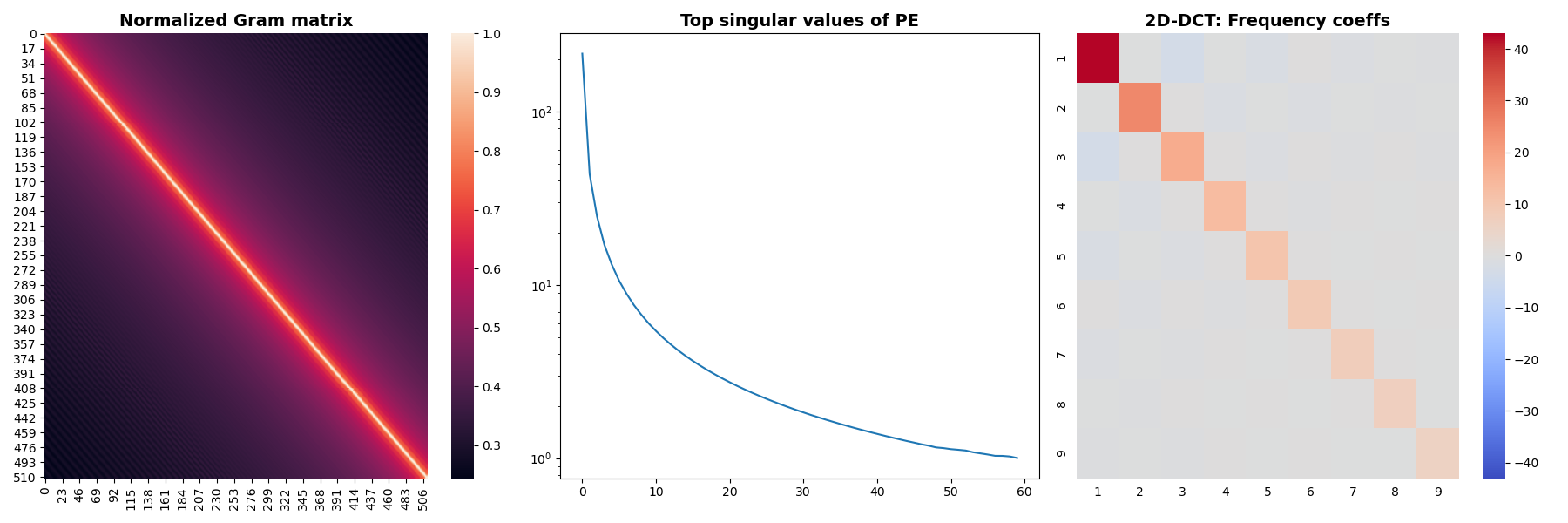}
\caption{\textbf{Original sinusoidal positional embedding} proposed in \citet{vaswani2017attention} is fixed. Compare the spectral and Fourier plots with Figure~\ref{fig:svd} and Figure~\ref{fig:Gram}.}
\label{fig:PE-compare}
\end{figure}

The original sinusoidal positional embedding (PE) proposed in \citet{vaswani2017attention} is fixed:
\begin{equation*}
    [\pos_t]_i = \begin{cases}
        \sin \left( \frac{t}{10000^{2i/d}} \right),  & i = 2k \\ \cos \left( \frac{t}{10000^{2i/d}} \right).  & i = 2k+1
    \end{cases}
\end{equation*}
Although this PE and learned positional basis both have sinusoidal patterns, we observe differences by comparing Figure~\ref{fig:PE-compare} and Figure~\ref{fig:svd} and~\ref{fig:Gram}:
\begin{enumerate}
    \item The Gram matrix of learned positional basis is visually smooth, whereas the original PE has large spiky values in the diagonal.
    \item The learned positional basis has a more pronounced low-rank structure, as there is a sharper drop of top singular values.
    \item The learned positional basis has more low-frequency 2D Fourier coefficients.
\end{enumerate}

\subsection{Do residual connections explain consistent geometry?}\label{sec:append-not-residual}

\begin{figure}[h]
\centering
\includegraphics[width=0.7\textwidth]{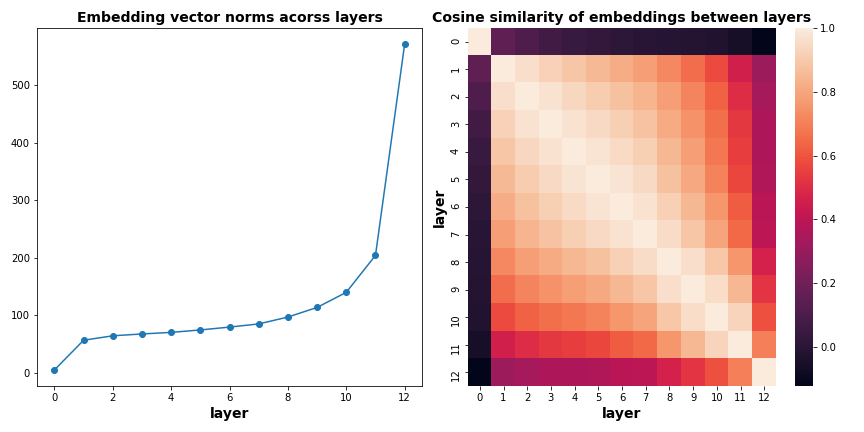}
\caption{\textbf{Average embedding norm and cosine similarity} from GPT-2 based on a random sampled sequence.}
\label{fig:embedding_norms}
\end{figure}

The residual connections in transformers provide an identity mapping across layers, which may seem to be a plausible reason for the consistency of observed geometry in Section~\ref{sec:geometry}.

We show evidence that this is not the case. We use pretrained GPT-2 to generate embeddings by feeding the transformer with random tokens of length $512$ uniformly sampled from the vocabulary. Then, we calculate the average norm of embeddings in each layer, and the cosine similarity of embeddings between every two layers.

In Figure~\ref{fig:embedding_norms}, we find
\begin{enumerate}
    \item The embedding norms increase significantly across layers. In particular, the norm ratio between Layer 1 and Layer 0 is 10.92, and between final layer and Layer 0 is 109.8.
    \item The cosine similarity between the $0$-th layer and any other layer is close to $0$, which indicates that embeddings experience dramatic changes from Layer 0 to Layer 1.
\end{enumerate}
These observations imply that residual connections cannot explain the consistent geometry in Section~\ref{sec:geometry}.

%% file: Appendix/SectionC.tex
\section{Additional empirical results for Section~\ref{sec:pos}}
We provide visualization and analysis for models other than GPT-2. 


\subsection{PCA visualization}

\input{Appendix/figures_pca}

See Figure \ref{fig:pca-github-gpt2}---Figure \ref{fig:pca-github-llama2}. Note: BERT displays a more complex circular shape, likely because its training objective is different from the others. 


\subsection{Low rank measurements}\label{sec:append-lowrank}
We provide rank estimates for all layers.

\paragraph{Rank estimate.} We report the rank estimate for all pretrained models and datasets in Table \ref{tab:screenot}. Additionally, we include another rank estimate---the Stable rank estimate \citep{rudelson2007sampling} in Table \ref{tab:stable-rank}.
\input{Appendix/table_low_rank}
\input{Appendix/table_stablerank}

\paragraph{Relative norm.} We report the relative norm for all pretrained models and datasets in Table \ref{tab:relative-norm}. 
\input{Appendix/table_relative_norm}






\subsection{On cluster structure}


\input{Appendix/figures_gram}

See Figure \ref{fig:gram-openwebtext-bloom}---Figure \ref{fig:gram-github-llama2}. The bottom right part of the normalized Gram matrix, namely $\ctx$-$\ctx$ part, exhibit progressively clear block structures. Four blocks indicate that there are $4$ clusters among all $\ctx$ vectors---which correspond to exactly $4$ documents we sample the sequences from.






%% file: Appendix/figures_pca.tex
\begin{figure}[p]
\centering
\includegraphics[width=0.98\textwidth]{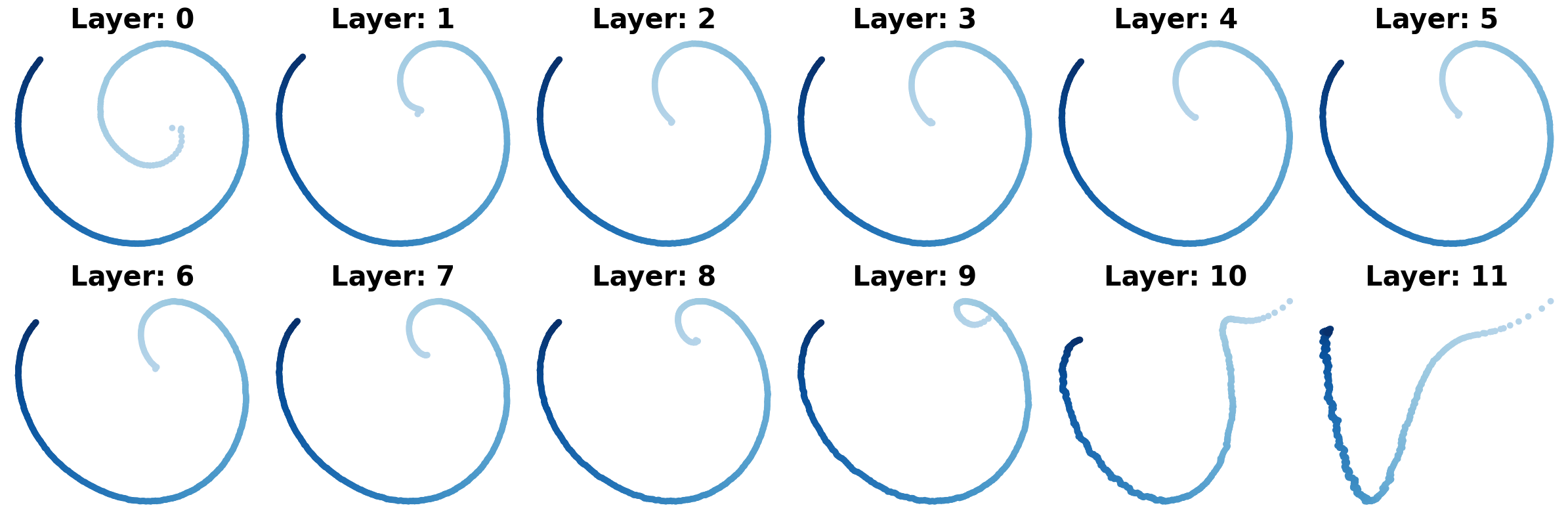}
\caption{Top-2 principal components of positional basis; GitHub, GPT2}
\label{fig:pca-github-gpt2}
\end{figure}

\begin{figure}[p]
\centering
\includegraphics[width=0.98\textwidth]{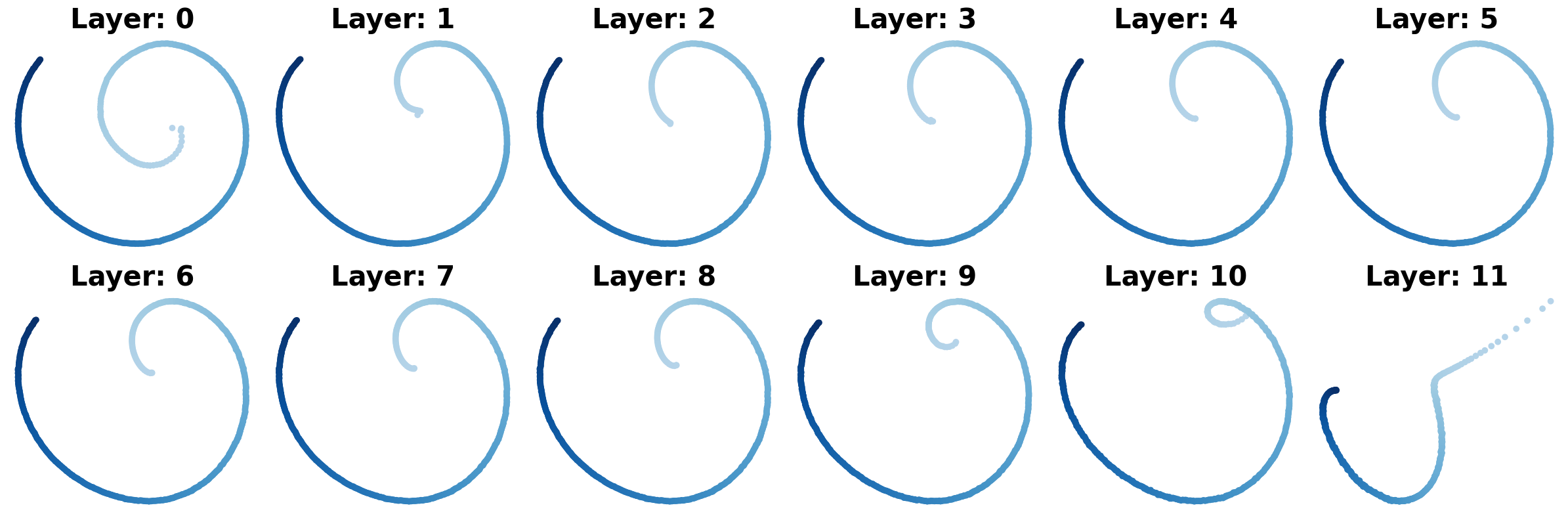}
\caption{Top-2 principal components of positional basis; WikiText, GPT2}
\label{fig:pca-wikitext-gpt2}
\end{figure}

\begin{figure}[p]
\centering
\includegraphics[width=0.98\textwidth]{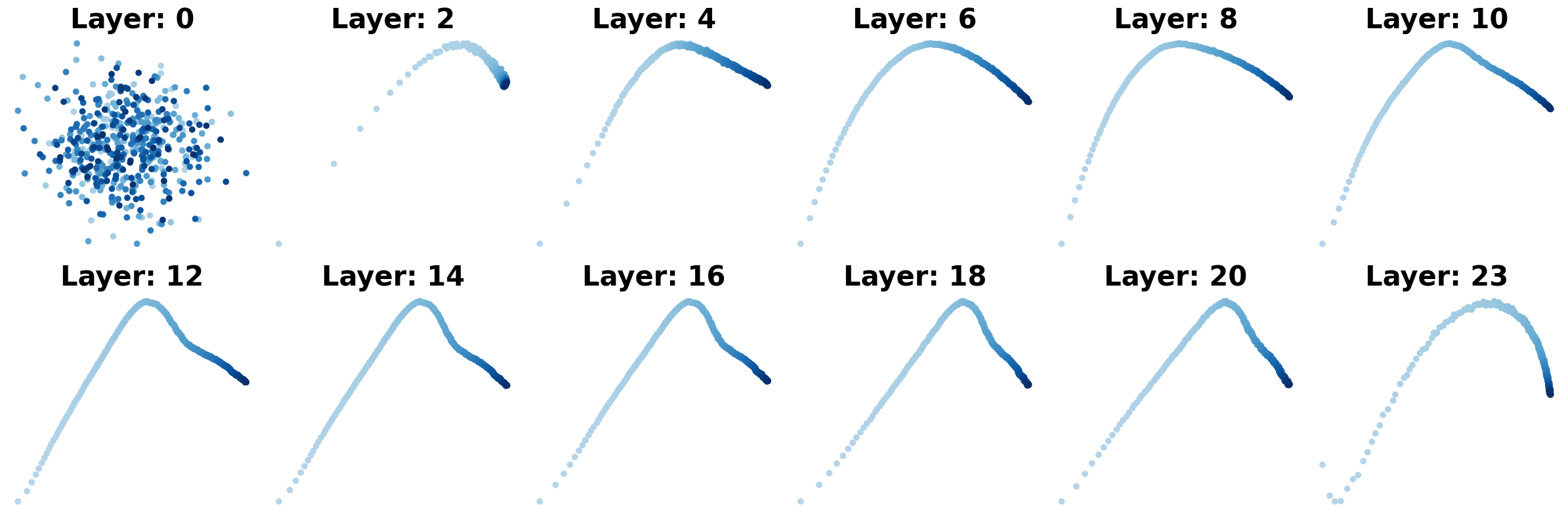}
\caption{Top-2 principal components of positional basis; OpenWebText, BLOOM}
\label{fig:pca-openwebtext-bloom}
\end{figure}

\begin{figure}[p]
\centering
\includegraphics[width=0.98\textwidth]{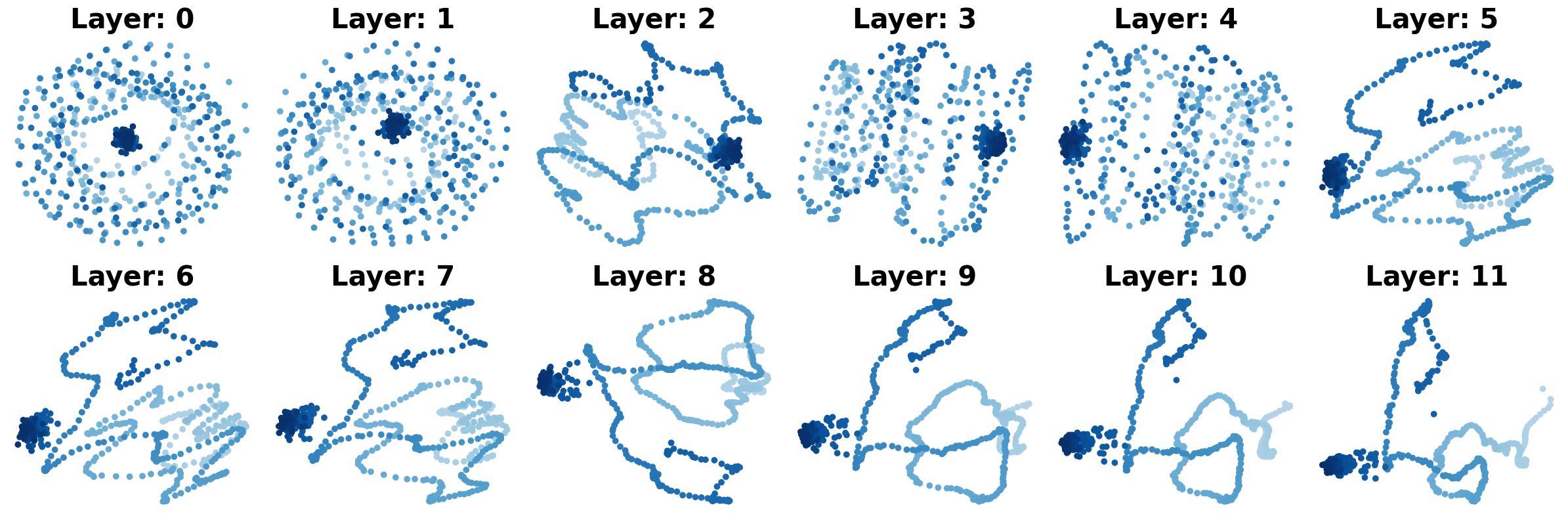}
\caption{Top-2 principal components of positional basis; OpenWebText, BERT}
\label{fig:pca-openwebtext-bert}
\end{figure}

\begin{figure}[p]
\centering
\includegraphics[width=0.98\textwidth]{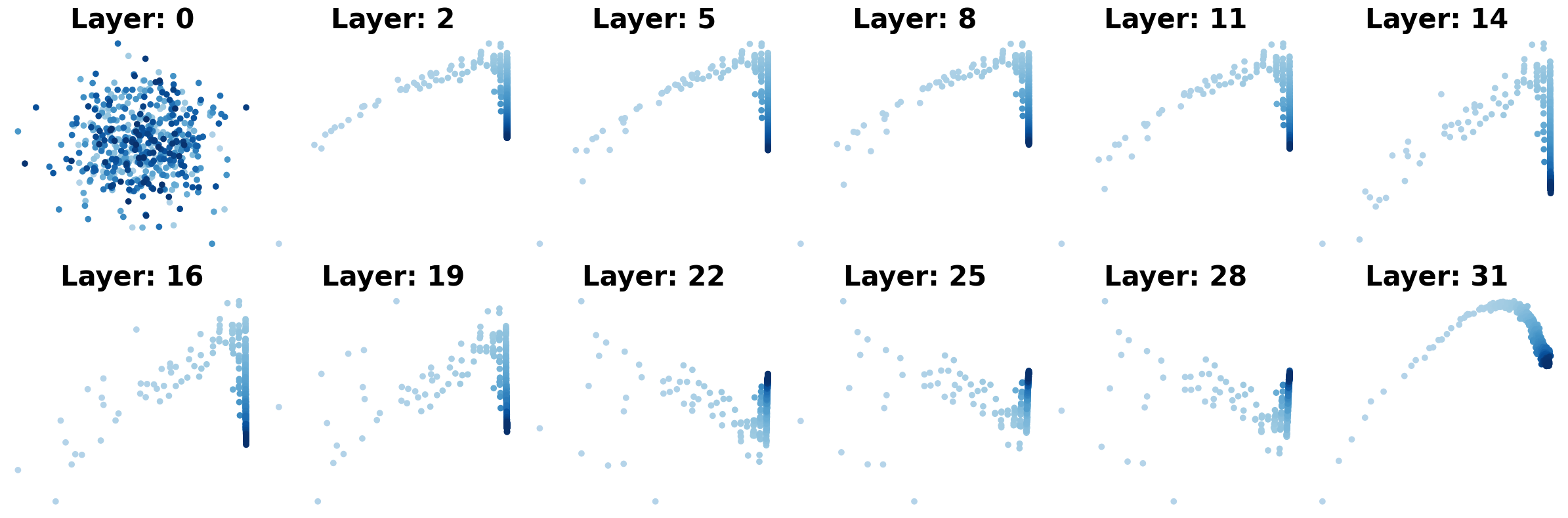}
\caption{Top-2 principal components of positional basis; OpenWebText, Llama2}
\label{fig:pca-openwebtext-llama2}
\end{figure}

\begin{figure}[p]
\centering
\includegraphics[width=0.98\textwidth]{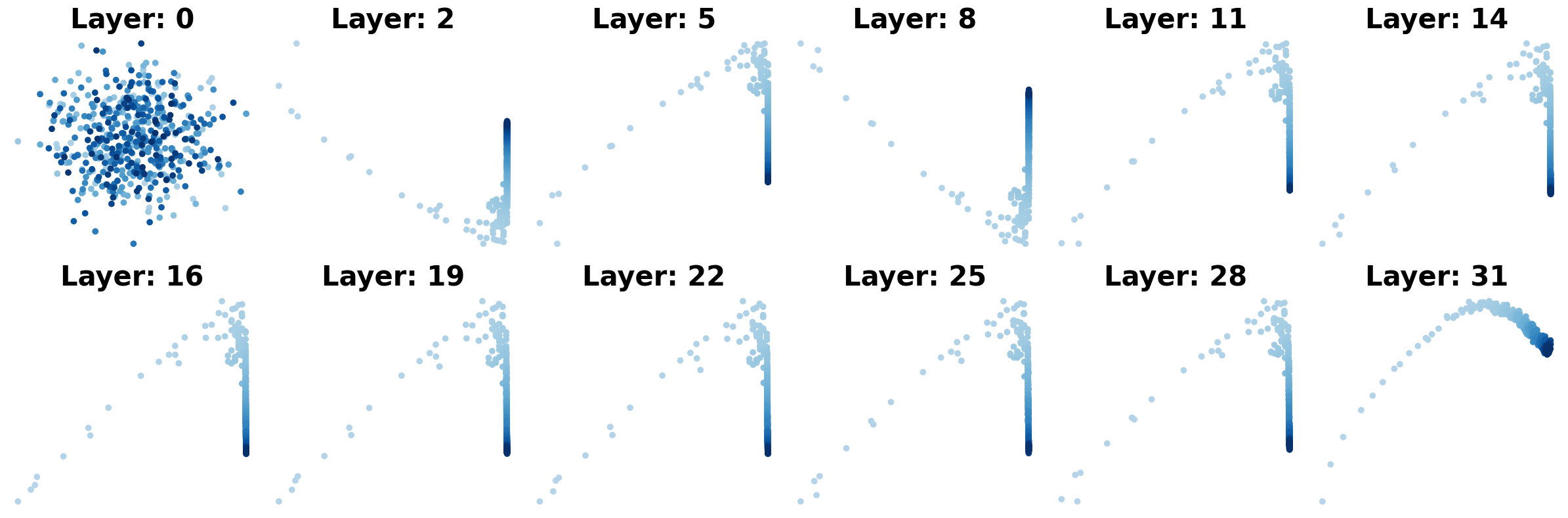}
\caption{Top-2 principal components of positional basis; GitHub, Llama2}
\label{fig:pca-github-llama2}
\end{figure}

%% file: Appendix/table_low_rank.tex
\begin{table}
\caption{ScreeNOT Rank Estimate for models, datasets and at each layer.}
\label{tab:screenot}
\resizebox{\columnwidth}{!}{%
\begin{tabular}{@{}llrrrrrrrrrrrrr@{}}
\toprule
 &
   &
  \textbf{Layer 0} &
  \textbf{Layer 1} &
  \textbf{Layer 2} &
  \textbf{Layer 3} &
  \textbf{Layer 4} &
  \textbf{Layer 5} &
  \textbf{Layer 6} &
  \textbf{Layer 7} &
  \textbf{Layer 8} &
  \textbf{Layer 9} &
  \textbf{Layer 10} &
  \textbf{Layer 11} &
  \textbf{Layer 12} \\ \midrule
\multirow{3}{*}{\textbf{BERT}}   & \textbf{GitHub}      & 15 & 16 & 16 & 16 & 14 & 11 & 11 & 9  & 10 & 10 & 11 & 11 & 12 \\
                                 & \textbf{OpenWebText} & 15 & 16 & 18 & 16 & 11 & 11 & 9  & 9  & 11 & 11 & 11 & 11 & 13 \\
                                 & \textbf{WikiText}    & 15 & 16 & 18 & 16 & 12 & 11 & 9  & 9  & 11 & 11 & 11 & 12 & 12 \\
\multirow{3}{*}{\textbf{BLOOM}}  & \textbf{GitHub}      & 8  & 9  & 9  & 8  & 9  & 10 & 10 & 11 & 10 & 10 & 10 & 10 & 10 \\
                                 & \textbf{OpenWebText} & 6  & 10 & 10 & 11 & 11 & 10 & 11 & 11 & 11 & 11 & 10 & 10 & 11 \\
                                 & \textbf{WikiText}    & 6  & 8  & 9  & 10 & 10 & 11 & 11 & 11 & 11 & 11 & 11 & 10 & 11 \\
\multirow{3}{*}{\textbf{GPT2}}   & \textbf{GitHub}      & 15 & 14 & 13 & 12 & 12 & 11 & 11 & 10 & 10 & 10 & 11 & 11 & 10 \\
                                 & \textbf{OpenWebText} & 15 & 13 & 14 & 12 & 13 & 11 & 10 & 10 & 10 & 10 & 9  & 9  & 12 \\
                                 & \textbf{WikiText}    & 15 & 14 & 14 & 12 & 11 & 11 & 11 & 11 & 11 & 11 & 9  & 10 & 12 \\
\multirow{3}{*}{\textbf{Llama2}} & \textbf{GitHub}      & 6  & 10 & 9  & 8  & 10 & 8  & 8  & 9  & 9  & 9  & 9  & 8  & 10 \\
                                 & \textbf{OpenWebText} & 7  & 10 & 10 & 11 & 11 & 10 & 9  & 10 & 9  & 8  & 9  & 8  & 10 \\
                                 & \textbf{WikiText}    & 8  & 10 & 10 & 10 & 9  & 8  & 8  & 8  & 8  & 8  & 8  & 8  & 10 \\ \bottomrule
\end{tabular}%
}
\end{table}

%% file: Appendix/table_stablerank.tex
\begin{table}
\centering
\caption{Stable rank for models, datasets and at each layer.}
\label{tab:stable-rank}
\resizebox{\columnwidth}{!}{%
\begin{tabular}{@{}llrrrrrrrrrrrrr@{}}
\toprule
 &
   &
  \textbf{Layer 0} &
  \textbf{Layer 1} &
  \textbf{Layer 2} &
  \textbf{Layer 3} &
  \textbf{Layer 4} &
  \textbf{Layer 5} &
  \textbf{Layer 6} &
  \textbf{Layer 7} &
  \textbf{Layer 8} &
  \textbf{Layer 9} &
  \textbf{Layer 10} &
  \textbf{Layer 11} &
  \textbf{Layer 12} \\ \midrule
\multirow{3}{*}{\textbf{BERT}}   & \textbf{GitHub}      & 9.19  & 7.79 & 5.26 & 4.73 & 4.34 & 3.84 & 3.48 & 3.20 & 2.70 & 2.45 & 2.04 & 1.84 & 1.91 \\
                                 & \textbf{OpenWebText} & 9.19  & 7.63 & 5.25 & 4.73 & 4.10 & 3.53 & 3.16 & 2.84 & 2.46 & 2.30 & 2.18 & 2.22 & 2.15 \\
                                 & \textbf{WikiText}    & 9.19  & 7.78 & 5.03 & 4.58 & 3.99 & 3.48 & 3.14 & 2.82 & 2.42 & 2.27 & 2.13 & 2.16 & 2.12 \\
\multirow{3}{*}{\textbf{BLOOM}}  & \textbf{GitHub}      & 8.39  & 1.25 & 1.20 & 1.21 & 1.21 & 1.23 & 1.29 & 1.29 & 1.28 & 1.25 & 1.21 & 1.02 & 1.00 \\
                                 & \textbf{OpenWebText} & 8.33  & 1.27 & 1.30 & 1.24 & 1.24 & 1.27 & 1.32 & 1.34 & 1.33 & 1.26 & 1.16 & 1.01 & 1.00 \\
                                 & \textbf{WikiText}    & 8.42  & 1.27 & 1.28 & 1.30 & 1.31 & 1.34 & 1.41 & 1.43 & 1.41 & 1.32 & 1.22 & 1.01 & 1.00 \\
\multirow{3}{*}{\textbf{GPT2}}   & \textbf{GitHub}      & 2.05  & 1.92 & 1.91 & 1.89 & 1.90 & 1.90 & 1.92 & 1.94 & 1.98 & 2.03 & 2.05 & 1.70 & 1.11 \\
                                 & \textbf{OpenWebText} & 2.05  & 1.92 & 1.91 & 1.89 & 1.88 & 1.88 & 1.88 & 1.90 & 1.91 & 1.96 & 2.02 & 2.24 & 1.49 \\
                                 & \textbf{WikiText}    & 2.05  & 1.92 & 1.91 & 1.89 & 1.88 & 1.88 & 1.88 & 1.90 & 1.91 & 1.97 & 2.03 & 2.19 & 1.56 \\
\multirow{3}{*}{\textbf{Llama2}} & \textbf{GitHub}      & 24.87 & 1.00 & 1.00 & 1.00 & 1.00 & 1.00 & 1.00 & 1.01 & 1.01 & 1.01 & 1.02 & 1.03 & 1.17 \\
                                 & \textbf{OpenWebText} & 52.23 & 1.00 & 1.00 & 1.00 & 1.00 & 1.00 & 1.01 & 1.01 & 1.02 & 1.02 & 1.03 & 1.05 & 1.44 \\
                                 & \textbf{WikiText}    & 24.70 & 1.00 & 1.00 & 1.01 & 1.01 & 1.02 & 1.03 & 1.05 & 1.09 & 1.16 & 1.20 & 1.26 & 1.30 \\ \bottomrule
\end{tabular}%
}
\end{table}

%% file: Appendix/table_relative_norm.tex
\begin{table}[]
\caption{Relative norm for models, datasets and at each layer.}
\label{tab:relative-norm}
\resizebox{\columnwidth}{!}{%
\begin{tabular}{@{}llrrrrrrrrrrrrr@{}}
\toprule
 &
   &
  \textbf{Layer 0} &
  \textbf{Layer 1} &
  \textbf{Layer 2} &
  \textbf{Layer 3} &
  \textbf{Layer 4} &
  \textbf{Layer 5} &
  \textbf{Layer 6} &
  \textbf{Layer 7} &
  \textbf{Layer 8} &
  \textbf{Layer 9} &
  \textbf{Layer 10} &
  \textbf{Layer 11} &
  \textbf{Layer 12} \\ \midrule
\multirow{3}{*}{\textbf{BERT}}   & \textbf{GitHub}      & 0.445 & 0.483 & 0.569 & 0.616 & 0.648 & 0.707 & 0.764 & 0.786 & 0.768 & 0.686 & 0.631 & 0.562 & 0.473 \\
                                 & \textbf{OpenWebText} & 0.465 & 0.546 & 0.660 & 0.759 & 0.877 & 0.977 & 0.973 & 0.967 & 0.953 & 0.901 & 0.777 & 0.658 & 0.596 \\
                                 & \textbf{WikiText}    & 0.454 & 0.502 & 0.626 & 0.695 & 0.798 & 0.916 & 0.968 & 0.965 & 0.949 & 0.887 & 0.756 & 0.682 & 0.627 \\
\multirow{3}{*}{\textbf{BLOOM}}  & \textbf{GitHub}      & 0.013 & 0.123 & 0.232 & 0.279 & 0.343 & 0.385 & 0.343 & 0.306 & 0.301 & 0.306 & 0.325 & 0.219 & 0.181 \\
                                 & \textbf{OpenWebText} & 0.012 & 0.138 & 0.194 & 0.264 & 0.315 & 0.342 & 0.297 & 0.267 & 0.264 & 0.300 & 0.392 & 0.575 & 0.589 \\
                                 & \textbf{WikiText}    & 0.013 & 0.149 & 0.222 & 0.287 & 0.328 & 0.352 & 0.336 & 0.301 & 0.295 & 0.325 & 0.407 & 0.494 & 0.491 \\
\multirow{3}{*}{\textbf{GPT2}}   & \textbf{GitHub}      & 0.999 & 0.994 & 0.996 & 0.972 & 0.812 & 0.762 & 0.672 & 0.600 & 0.489 & 0.442 & 0.386 & 0.303 & 0.123 \\
                                 & \textbf{OpenWebText} & 1.000 & 0.996 & 0.990 & 0.989 & 0.986 & 0.983 & 0.981 & 0.979 & 0.974 & 0.841 & 0.631 & 0.480 & 0.075 \\
                                 & \textbf{WikiText}    & 1.000 & 0.995 & 0.991 & 0.989 & 0.987 & 0.986 & 0.984 & 0.984 & 0.981 & 0.933 & 0.680 & 0.555 & 0.110 \\
\multirow{3}{*}{\textbf{Llama2}} & \textbf{GitHub}      & 0.029 & 0.221 & 0.221 & 0.222 & 0.222 & 0.222 & 0.223 & 0.223 & 0.224 & 0.225 & 0.226 & 0.227 & 0.183 \\
                                 & \textbf{OpenWebText} & 0.038 & 0.146 & 0.146 & 0.146 & 0.147 & 0.147 & 0.148 & 0.148 & 0.150 & 0.152 & 0.154 & 0.154 & 0.148 \\
                                 & \textbf{WikiText}    & 0.034 & 0.060 & 0.060 & 0.060 & 0.060 & 0.061 & 0.062 & 0.063 & 0.065 & 0.068 & 0.072 & 0.076 & 0.191 \\ \bottomrule
\end{tabular}%
}
\end{table}

%% file: Appendix/figures_gram.tex
\begin{figure}[p]
\centering
\includegraphics[width=0.98\textwidth]{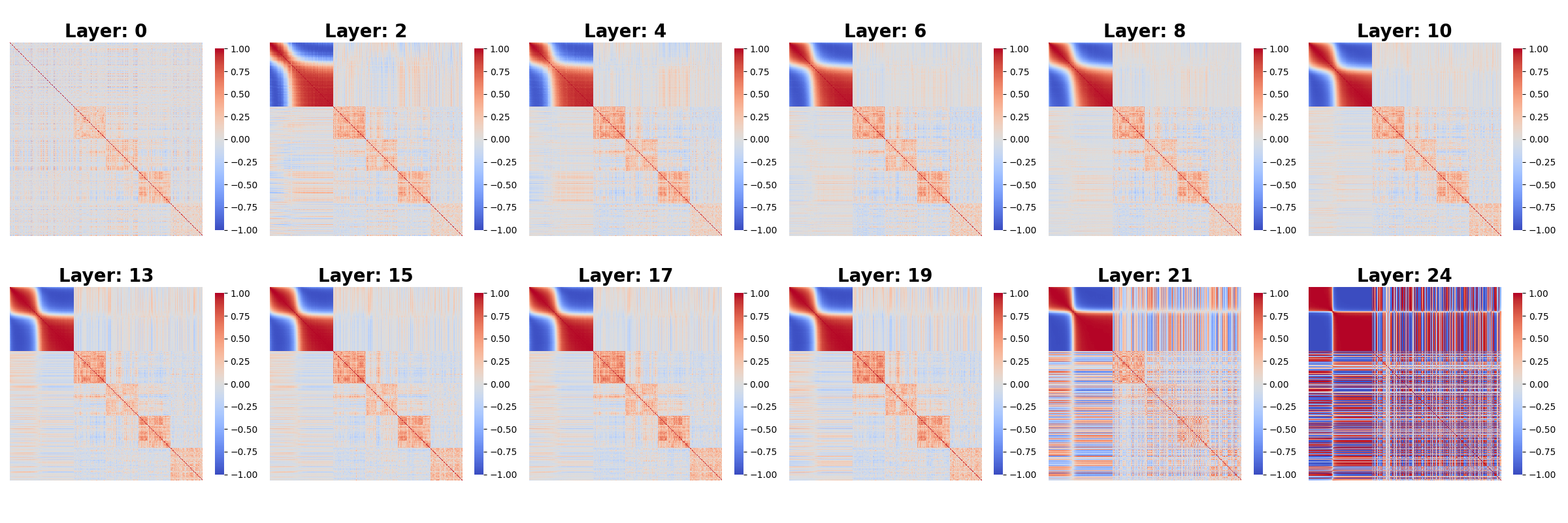}
\caption{Gram matrix of positional basis and context basis; Openwebtext, BLOOM}
\label{fig:gram-openwebtext-bloom}
\end{figure}

\begin{figure}[p]
\centering
\includegraphics[width=0.98\textwidth]{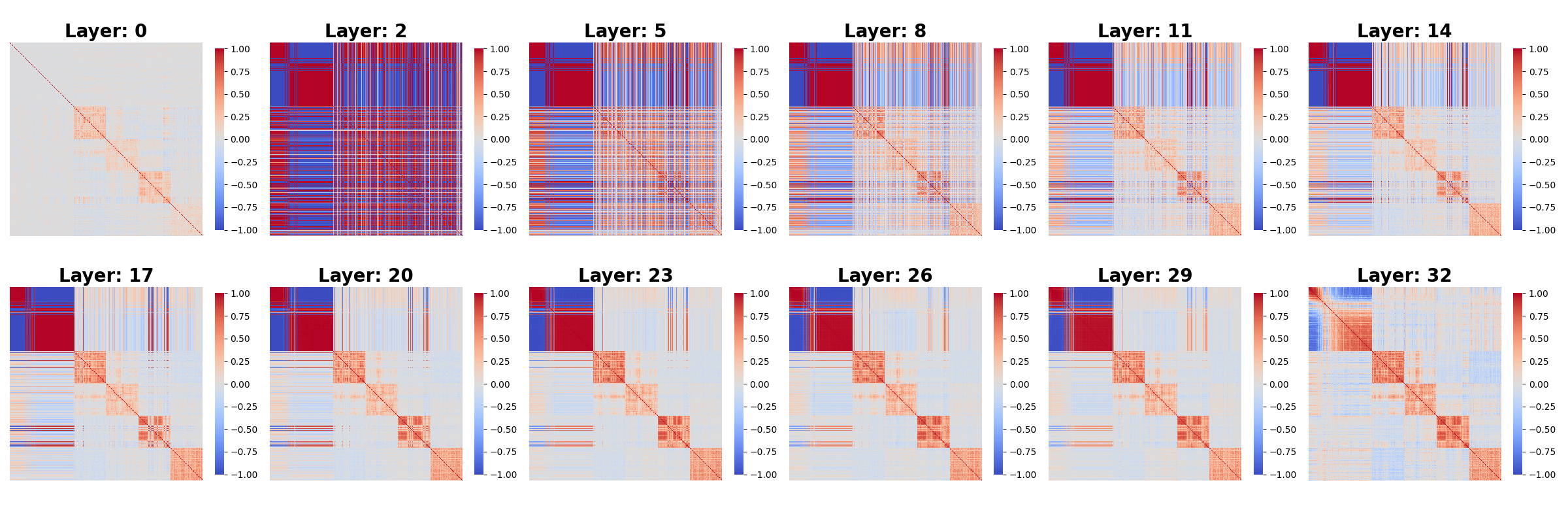}
\caption{Gram matrix of positional basis and context basis; Openwebtext, Llama2}
\label{fig:gram-openwebtext-llama2}
\end{figure}

\begin{figure}[p]
\centering
\includegraphics[width=0.98\textwidth]{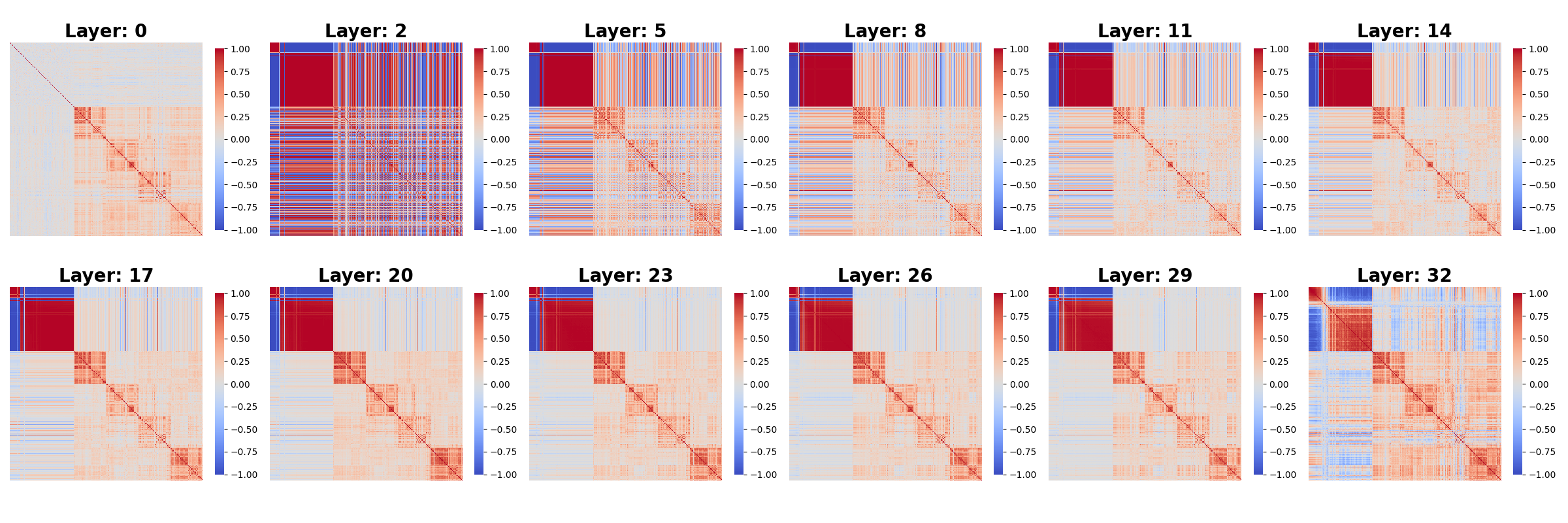}
\caption{Gram matrix of positional basis and context basis; GitHub, Llama}
\label{fig:gram-github-llama2}
\end{figure}

%% file: Appendix/SectionD.tex
\section{Additional empirical results for Section~\ref{sec:smoothness}} \label{sec:append-smoothness}





\input{Appendix/figures_addition}

\subsection{On robustness of positional basis}\label{sec:append-pos-ood}

\begin{figure}[p]
\centering
\includegraphics[width=0.98\textwidth]{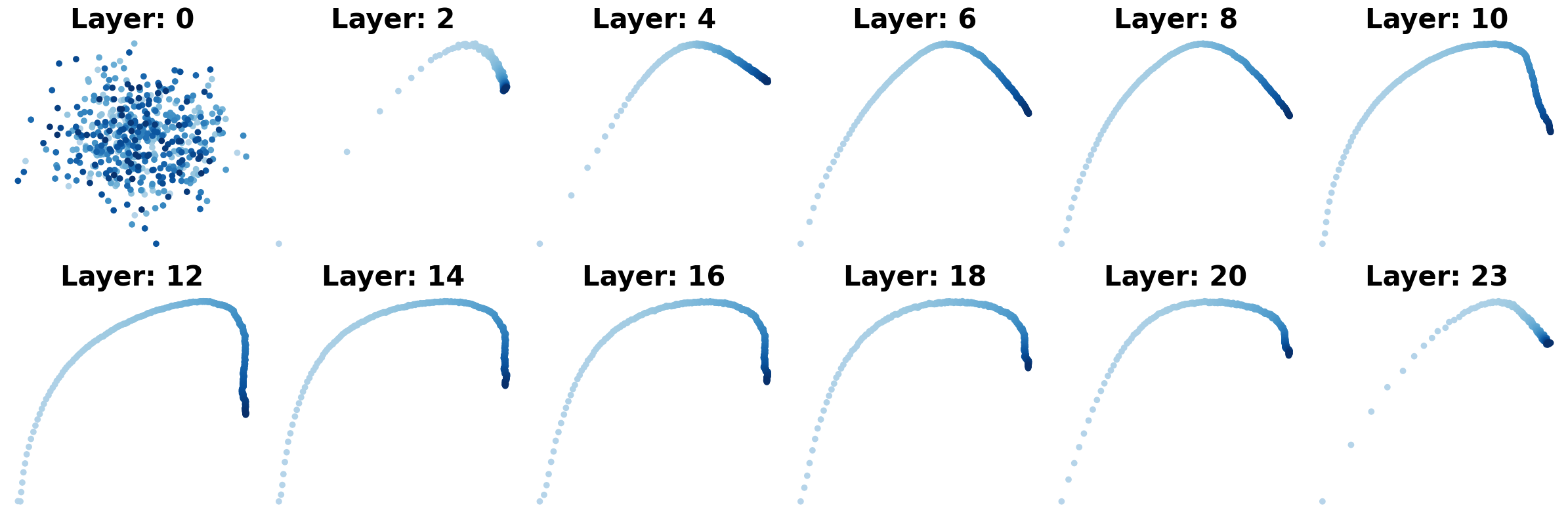}
\caption{Top-2 PC plot of positional basis ; Bloom on random tokens}
\label{fig:pos_bloom_random}
\end{figure}

In Table~\ref{tab:ood}, we sampled OOD sequences to test if positional basis is robust and reported the ratio explained by top-10 frequencies. 

Here we provide a visual examination, which further confirms that positional basis possesses similar low-rank and spiral structure. We choose BLOOM, which has been pretrained on both natural languages and programming languages, and we use random sampled tokens to test whether its associated positional basis still has similar structure.

We find that similar structure persists even for the above randomly sampled sequences. See Figure \ref{fig:pos_bloom_random}.

\subsection{On sparse and local attention}\label{sec:append-local-attn}

\begin{figure}[h]
\centering
\includegraphics[width=0.5\textwidth]{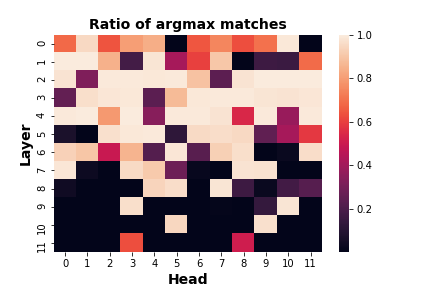}
\caption{\textbf{Ratio of positions that satisfy the argmax property \eqref{eq:argmax}}.}
\label{fig:argmax}
\end{figure}

Here we provide evidence that the argmax property \eqref{eq:argmax} is satisfied in GPT-2 at a substantial level. We consider a variant of $\pos$-$\pos$ QK constituent matrix $\tilde \mP \mW \tilde \mP^\top$ where each column of $\tilde \mP$ is $\pos_t + \vmu$, because we find the global mean $\vmu$ helps satisfy this argmax property. 

Note that checking this $\pos$-$\pos$ QK constituent instead of the QK matrix is more straightforward: $\pos$-$\pos$ QK constituent reflects the property of the model and positional basis, without needing the context information.

For each layer and each head, we calculate $\tilde \mP \mW \tilde \mP^\top$ and examine the fraction of positions $t$ such that 
\begin{equation*}
    \argmax_{t' \le t} [\tilde \mP \mW \tilde \mP^\top]_{t,t'} = t.
\end{equation*}
Figure~\ref{fig:argmax} shows the ratios for each layer and head. It is clear that this argmax property is mostly satisfied in many heads, especially among early layers.

\subsection{On Addition experiments}\label{sec:append-addition}


We have manually generated the addition dataset. The number of digits of the addition ranges from 5 to 10 , and is sampled under uniform distribution in the training process. The model achieve above 99\% accuracy on training set and in-distribution test set. However, the model does not achieve length generalization. It achieves 32.6\%, 8.0\%, 11.1\%, and 9.9\% accuracy on 1k samples of addition with digits length 1, 2, 3, 4 respectively. 

\paragraph{Similar findings for transformers with relative embedding.} Additionally, we implement a rotary-embedding based model by replacing the absolute positional embedding with RoPE \cite{su2021roformer}. For the rotary-embedding based model, the corresponding accuracy is 35.6\%, 13.4\%, 20.0\%, 38.5\%. 

\paragraph{Visualizing discontinuity.} We plot the QK matrices and Gram matrices for both transformers---with absolute embedding and with rotary position embedding. See~Figure \ref{fig:add-gram-abs}---Figure~\ref{fig:add-fourier-rotary}. Notice that the unsmoothness is pervasive in the Gram matrix of positional basis and QK matrix \textbf{across different layers, heads, and models}. 

\paragraph{Quantitative measurements.} We calculated the ratios that low-frequency components can explain in the Fourier space at various layers and heads. The lack of smoothness is further confirmed by the Fourier analysis.


%% file: Appendix/figures_addition.tex
\begin{figure}[p]
\centering
\includegraphics[width=0.98\textwidth]{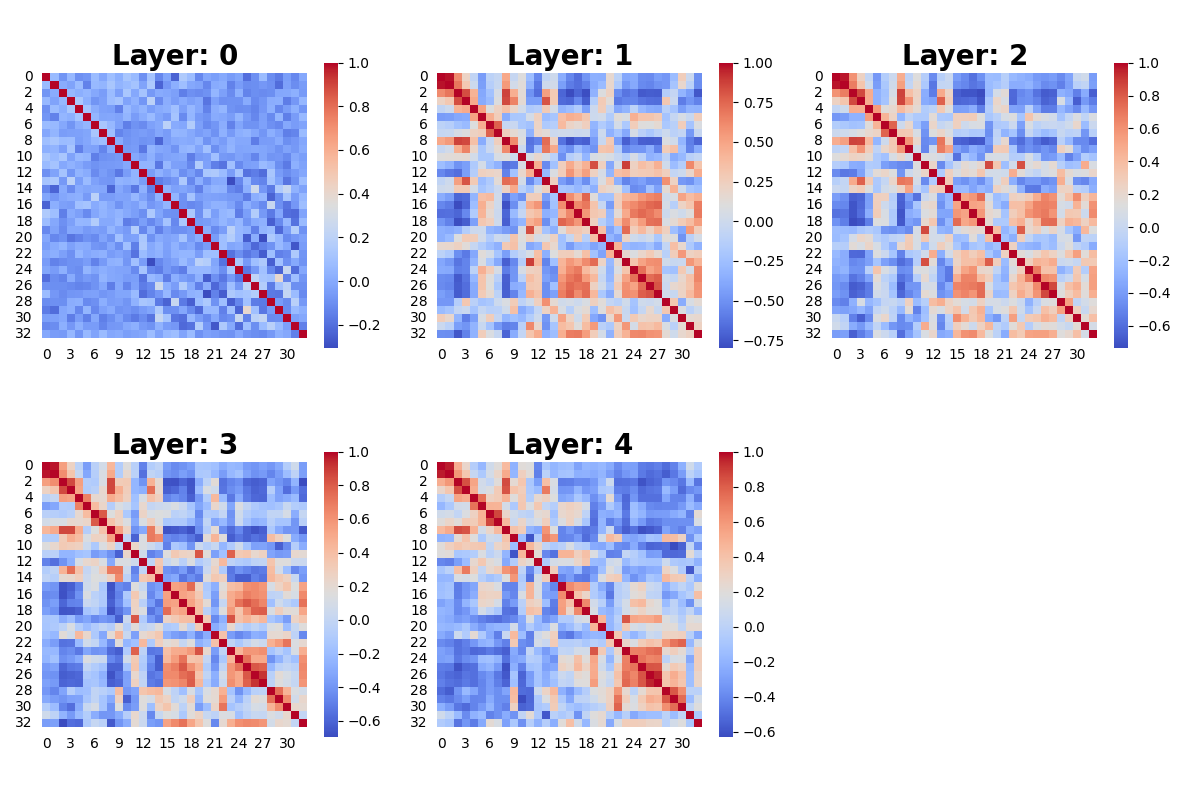}
\caption{Gram matrix of positional basis; \textbf{Addition} with absolute PE shows nonsmoothness.}
\label{fig:add-gram-abs}
\end{figure}

\begin{figure}[p]
\centering
\includegraphics[width=0.98\textwidth]{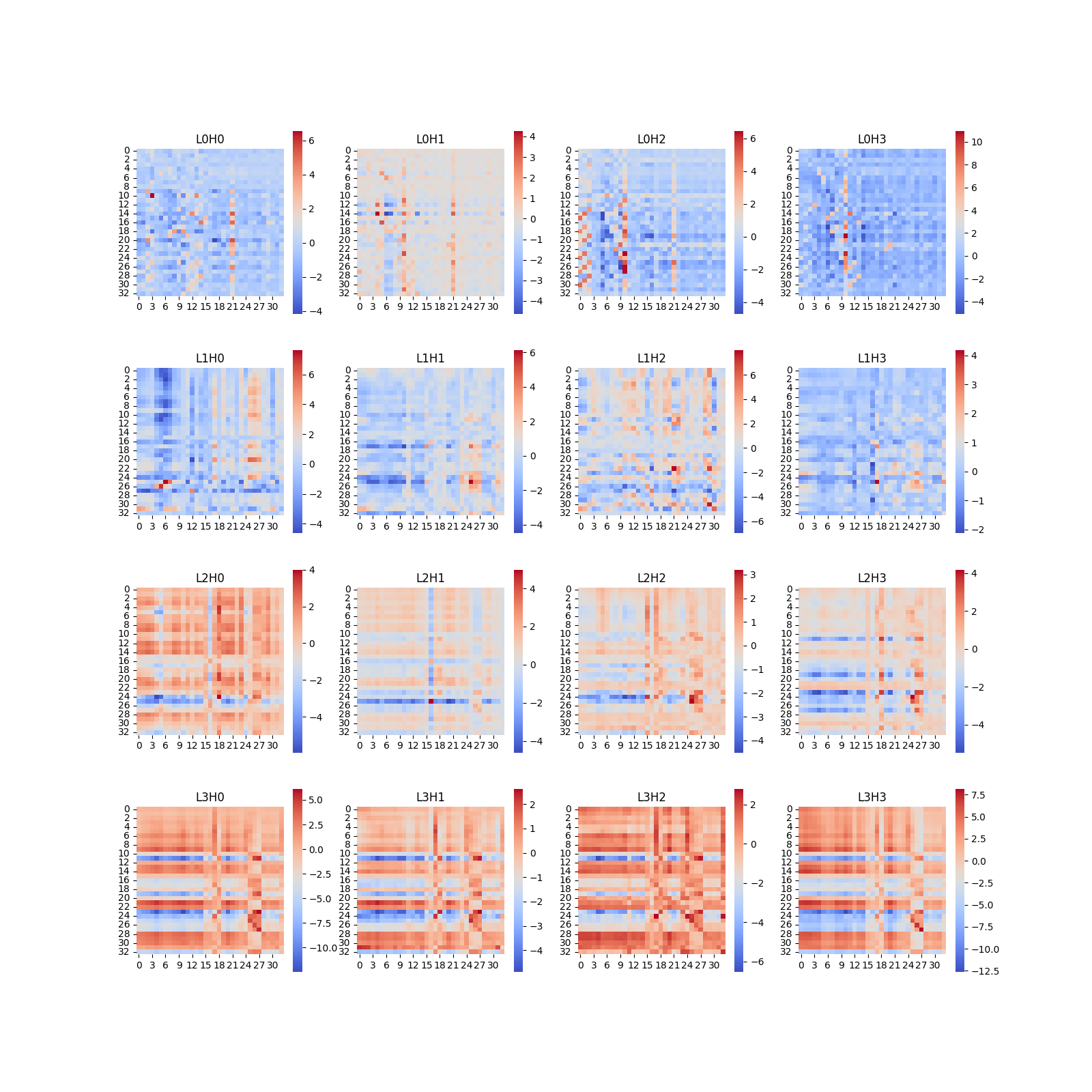}
\caption{QK matrix of positional basis; \textbf{Addition} with absolute PE shows nonsmoothness.}
\label{fig:add-qk-abs}
\end{figure}

\begin{figure}[p]
\centering
\includegraphics[width=0.98\textwidth]{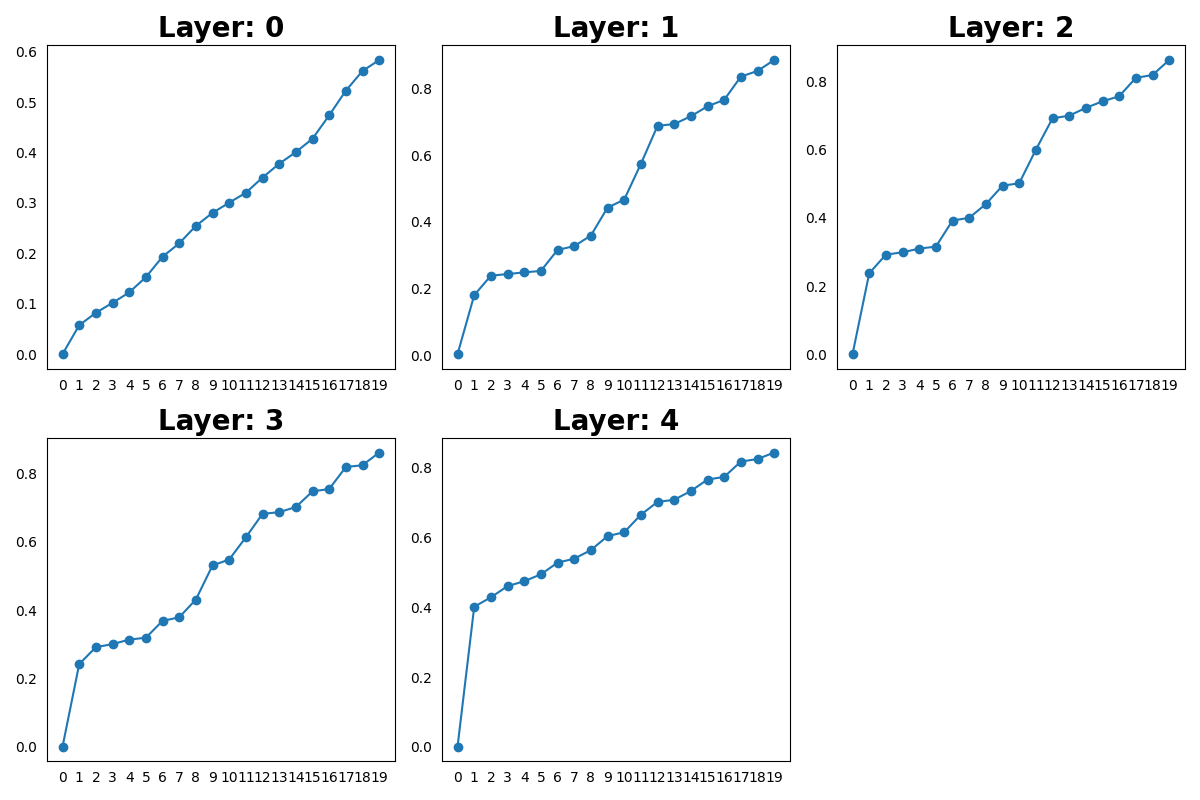}
\caption{Fourier of positional basis; \textbf{Addition} with absolute PE depends on higher-frequency components.}
\label{fig:add-fourier-abs}
\end{figure}

\begin{figure}[p]
\centering
\includegraphics[width=0.98\textwidth]{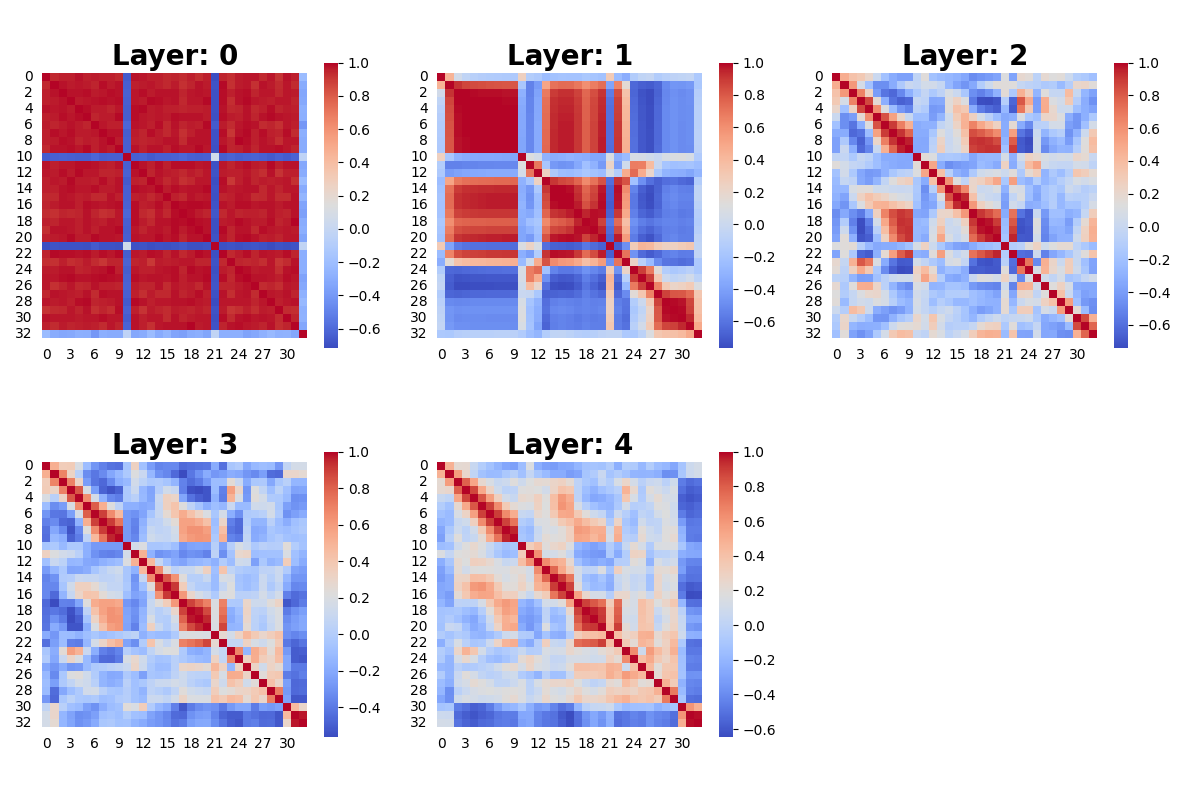}
\caption{Fourier of positional basis; \textbf{Addition} with rotary PE shows nonsmoothness.}
\label{fig:add-gram-rotary}
\end{figure}

\begin{figure}[p]
\centering
\includegraphics[width=0.98\textwidth]{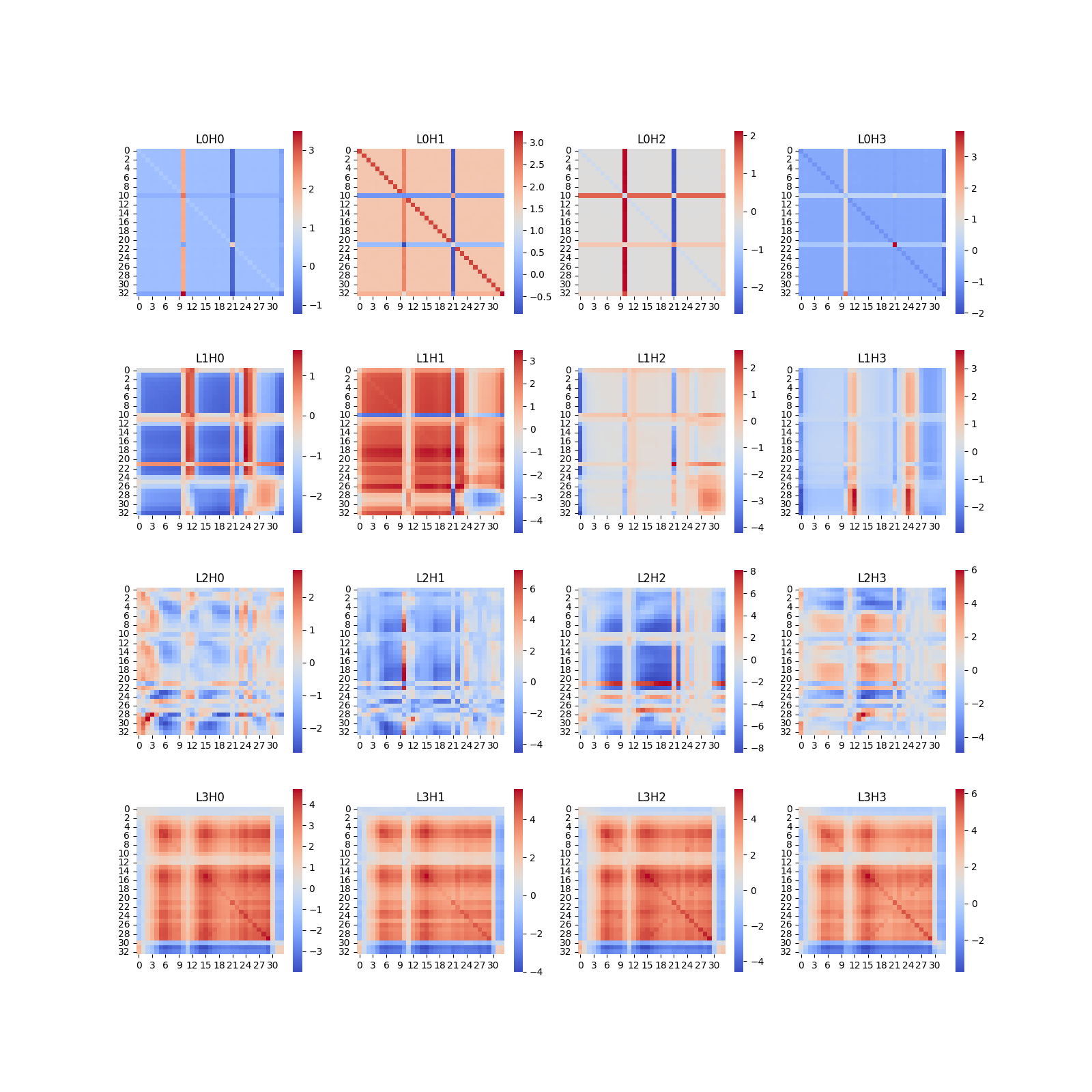}
\caption{QK matrix of positional basis; \textbf{Addition} with absolute PE shows nonsmoothness.}
\label{fig:add-qk-rotary}
\end{figure}

\begin{figure}[p]
\centering
\includegraphics[width=0.98\textwidth]{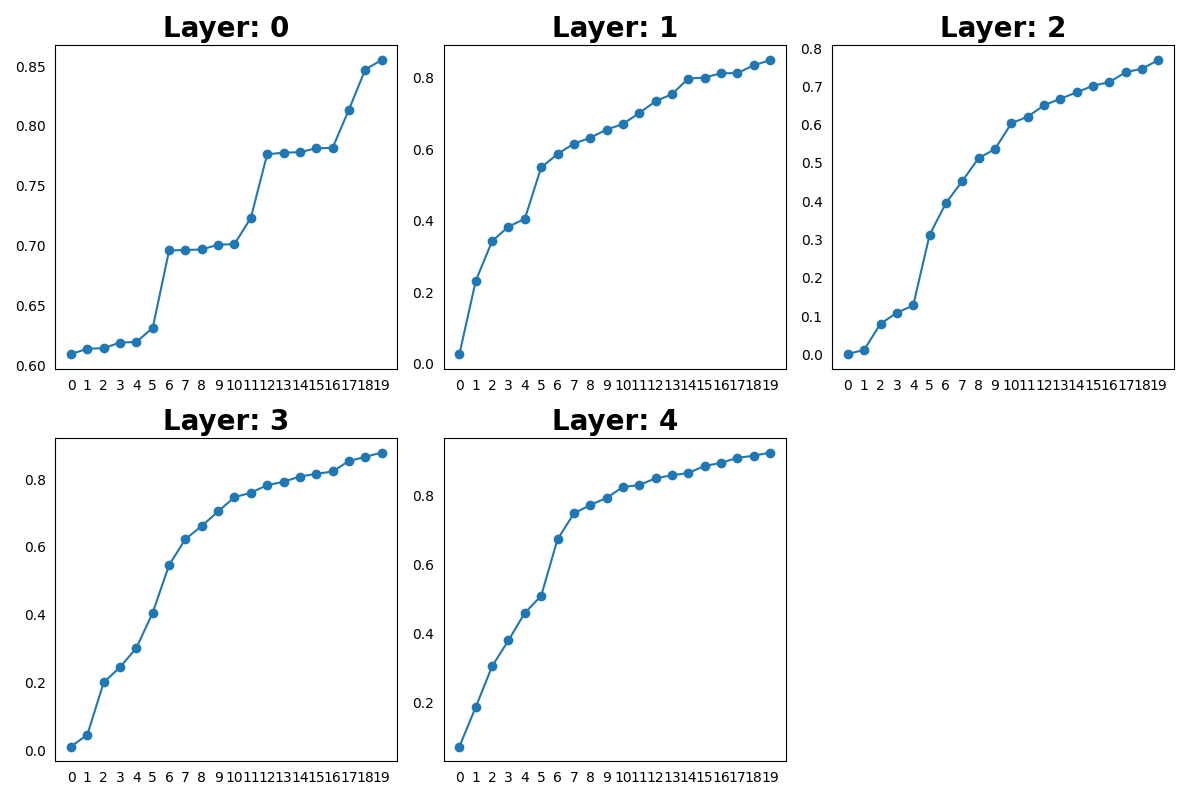}
\caption{Fourier of positional basis; \textbf{Addition} with absolute PE depends on higher-frequency components.}
\label{fig:add-fourier-rotary}
\end{figure}

%% file: Appendix/SectionE.tex
\section{Additional empirical results for Section~\ref{sec:incoh}}

\subsection{On trained weight matrix}\label{sec:append-attention-weight}

\input{Appendix/figures_att_weights}

We provide decomposition of trained weight matrices similar to Figure~\ref{fig:attention-weight}. See Figure~\ref{fig:attw-gpt2-L6}---Figure~\ref{fig:attw-bert-L10} for more plots about $\mW = \mW^q (\mW^k)^\top / \sqrt{d_\head}$ at various layers and heads for BERT and GPT-2 model.

\subsection{On enhanced QK/attention visualization}\label{sec:append-QK}

\begin{figure}[p]
\centering
\includegraphics[width=0.98\textwidth]{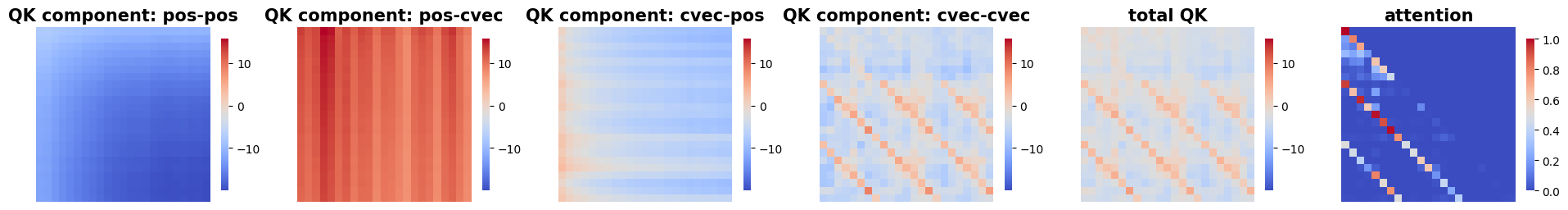}
\caption{QK matrix decomposition with global mean; GPT2, Layer 10 Head 7}
\label{fig:induction-head-gptL10H7}
\end{figure}

\begin{figure}[p]
\centering
\includegraphics[width=0.98\textwidth]{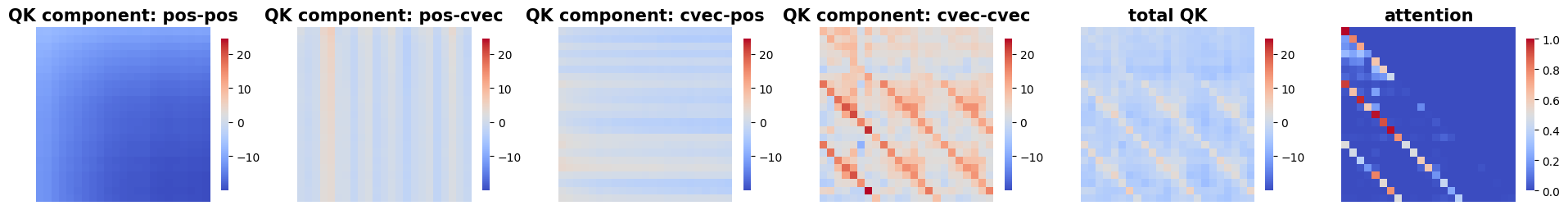}
\caption{QK matrix decomposition without global mean; GPT2, Layer 10 Head 7}
\label{fig:induction-head-gptL10H7-rm}
\end{figure}

\paragraph{On global mean vector.} We show the QK matrix decomposition with global mean (Figure \ref{fig:induction-head-gptL10H7}), and without global mean (Figure \ref{fig:induction-head-gptL10H7-rm}). Note that adding a constant to all entries of the QK matrix will not change the attention matrix, because softmax computes the ratio. We conclude that the global mean vector $\vmu$ has little effect on interpretations.



\input{Appendix/figures_qk}

\paragraph{Visualizing QK constituents.}
We use the QK decomposition \eqref{QKdecomp} to provide fine-grained visualization for attention heads. In Figure \ref{fig:induction-head-bloomL10H1}---Figure \ref{fig:induction-head-gptL4H11}, we select various heads on GPT-2 and BLOOM model and provide 6 subplots based on (i) the 4 QK constituents, (ii) the original QK matrix, and (iii) the attention matrix.

We find that visualization based on individual QK constituents provides additional information beyond visualizing the attention matrix alone. Our selection of heads is representative, as many similar heads appear in GPT-2 and BLOOM.

%% file: Appendix/figures_att_weights.tex
\begin{figure}[p]
\centering
\includegraphics[width=0.98\textwidth]{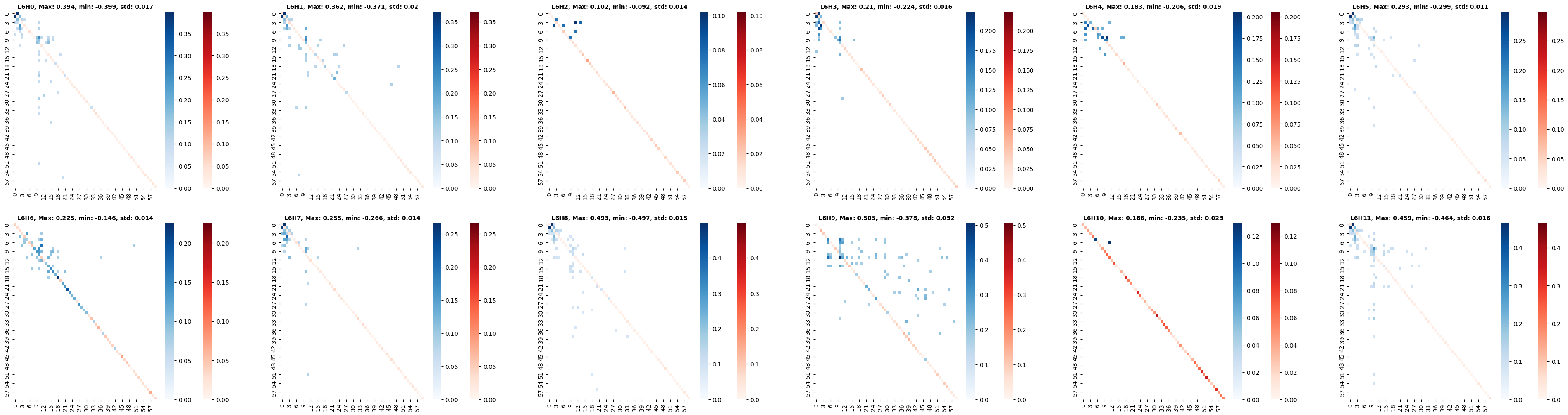}
\caption{Dissecting attention weights, GPT2 L6}
\label{fig:attw-gpt2-L6}
\end{figure}

\begin{figure}[p]
\centering
\includegraphics[width=0.98\textwidth]{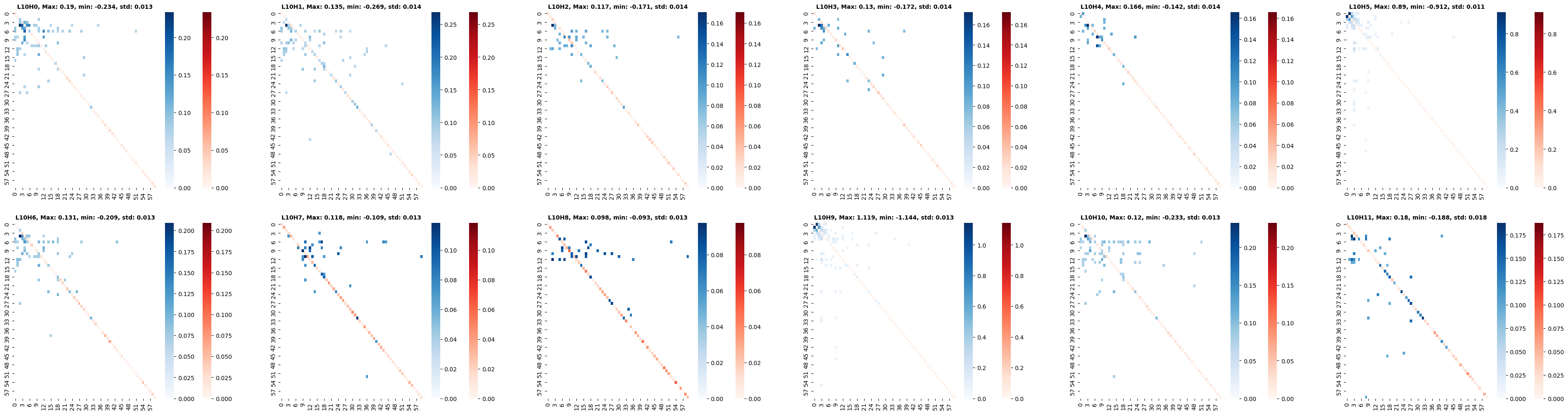}
\caption{Dissecting attention weights, GPT2 L10}
\label{fig:attw-gpt2-L10}
\end{figure}

\begin{figure}[p]
\centering
\includegraphics[width=0.98\textwidth]{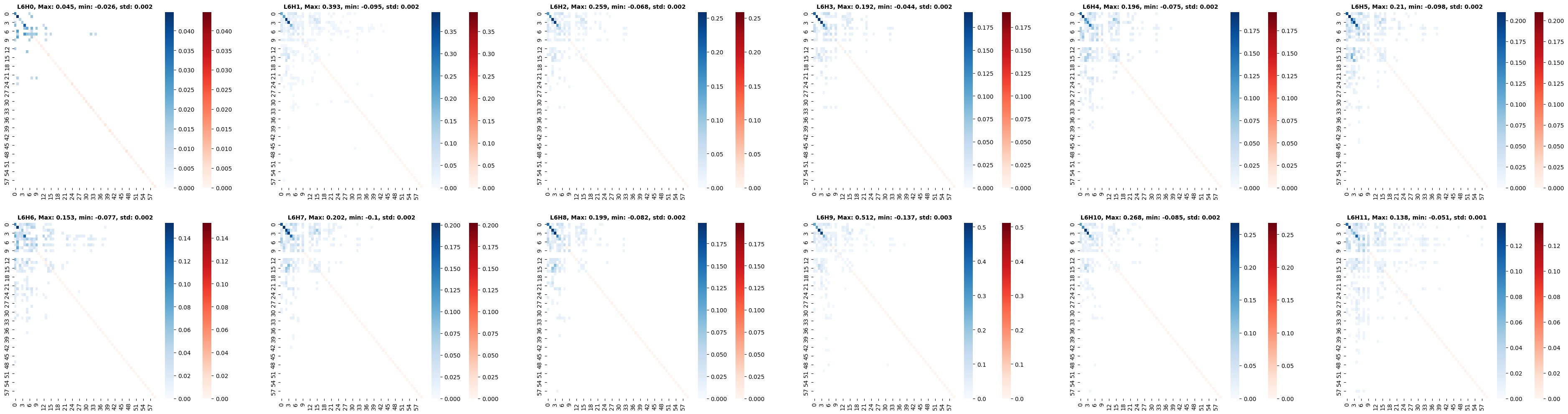}
\caption{Dissecting attention weights, BERT L6}
\label{fig:attw-bert-L6}
\end{figure}

\begin{figure}[p]
\centering
\includegraphics[width=0.98\textwidth]{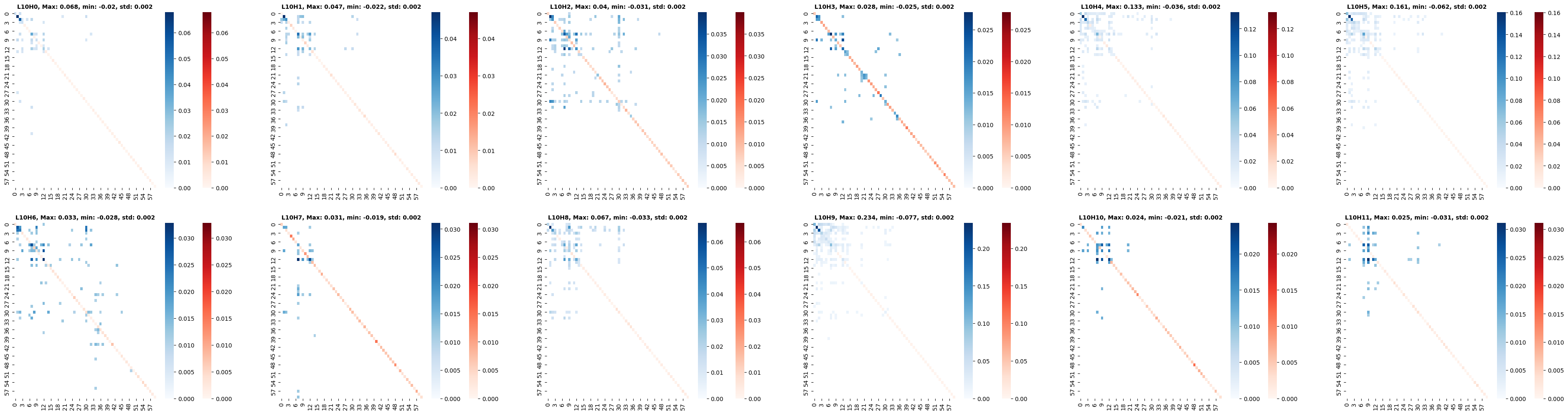}
\caption{Dissecting attention weights, BERT L10}
\label{fig:attw-bert-L10}
\end{figure}

%% file: Appendix/figures_qk.tex
\begin{figure}[p]
\centering
\includegraphics[width=0.98\textwidth]{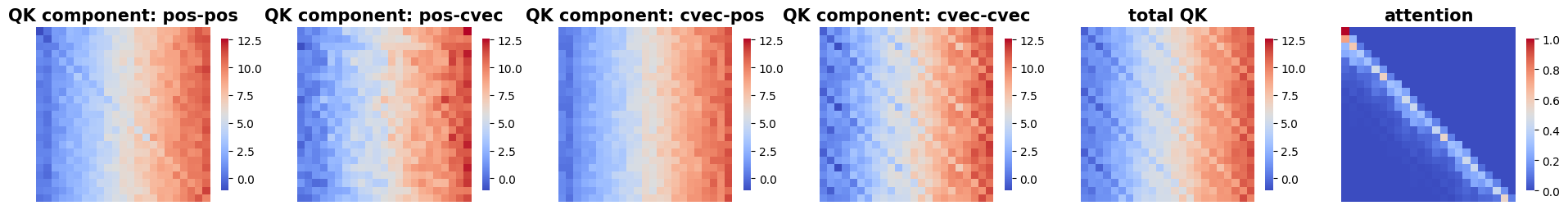}
\caption{QK Decomposition; BLOOM L10H1}
\label{fig:induction-head-bloomL10H1}
\end{figure}

\begin{figure}[p]
\centering
\includegraphics[width=0.98\textwidth]{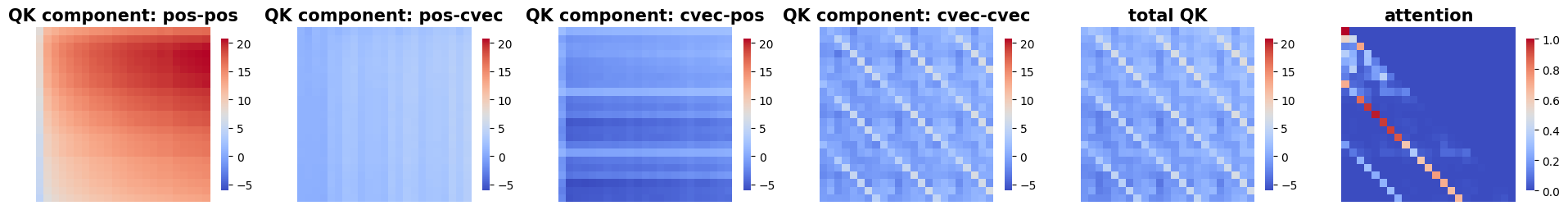}
\caption{QK Decomposition; BLOOM L7H6}
\label{fig:induction-head-bloomL7H6}
\end{figure}

\begin{figure}[p]
\centering
\includegraphics[width=0.98\textwidth]{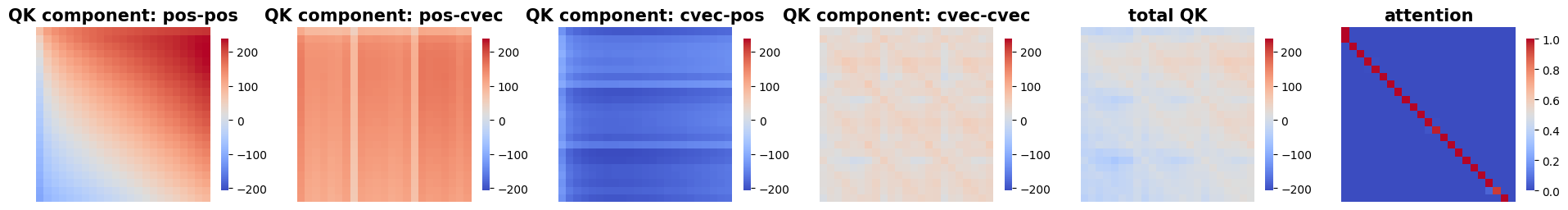}
\caption{QK Decomposition; GPT2 L4H11}
\label{fig:induction-head-gptL4H11}
\end{figure}

%% file: Appendix/SectionF.tex
\section{Proofs for theoretical results}\label{sec:append-proofs}

We introduce some additional notations. Denote the indicator function by $\vone$. We denote by $\vone_N$ the vector $(1,1,\ldots,1)^\top \in \R^N$. For a complex matrix $\mA$, we denote the conjugate transpose by $\mA^*$. For convenience, for a matrix $\mA$, we will write $\Delta \mA$ instead of $\Delta^{(1,1)} \mA$. For a vector $\vx \in \C^N$, we also write $\Delta \vx$ to denote the finite difference vector $N \cdot (x_1-x_0, x_2 - x_1, \ldots,  x_N - x_{N-1})^\top$ (where $x_N = x_0$). We will say that a Hermitian matrix $\mA \in \C^{N \times N}$ is positive semidefinite (PSD) if and only if $\vx^* \mA \vx \ge 0 $ for every $\vx \in \C^N$. Denote by $\Re(x)$ the real part of a complex number $\vx \in \C$.

\subsection{Proof of Theorem~\ref{thm:fourier}}\label{sec:proof-fourier}
In this subsection, we denote a generic dimension by $N$ and $\omega = \exp(-2\pi i / N)$.  We need some standard definitions and properties; see \citet{broughton2018discrete} for example. 

The discrete Fourier transform (DFT) matrix $\mF \in \C^{N \times N}$ is given by $F_{tt'} = \omega^{(t-1)(t'-1)}$ for $1 \le t,t' \le N $. The inverse discrete Fourier transform (IDFT) matrix is $N^{-1} \mF^*$. Both the DFT matrix and the IDFT matrix are symmetric (not Hermitian). Sometimes we prefer to write $\mF^\top$ instead of $\mF$ simply for formality. For a generic matrix $\mA \in \R^{N \times N}$, we denote by $\hat \mA \in \C^{N \times N}$ the matrix after its 2-d DFT. It satisfies
\begin{align}
\begin{split}\label{eq:DFT}
\hat \mA &= \mF \mA \mF^\top, \\
\mA &= N^{-2} \mF^* \hat \mA (\mF^*)^\top.
\end{split}
\end{align}

The following simple lemma is a consequence of integration-by-parts for the discrete version. For completeness we include a proof.
\begin{lem}\label{lem:fourier-ibp}
Let $\vx \in \R^{N}$ be a vector, and $\hat \vx = \mF \vx$ be its DFT. Then for $t=1,\ldots,N$,
\begin{align}
&\hat x_t = \gamma_t (\mF \Delta \vx)_t + \vone\{ t = 1\} \cdot \sum_{t'=1}^N x_{t'}, \qquad \text{where} \label{eq:hatx} \\
&\gamma_t := N^{-1} \left( 1- \exp \Big(\frac{-2\pi i (t-1)}{N} \Big) \right)^{-1} \quad \text{for} ~~t>1~~\text{and}~~\gamma_1:=1 . \label{def:gamma}
\end{align}
\end{lem}
\begin{proof}
If $t=1$, then $(\mF \vx)_t = \sum_{t'=1}^{N} x_{t'}$, and $(\mF \Delta \vx)_t = N \sum_{t'=1}^{N} (x_{t'} - x_{t'-1}) = 0$. For $t \neq 1$, 
\begin{align*}
(\mF \Delta \vx)_t &= N \sum_{t'=1}^N \omega^{(t-1)(t'-1)} (x_{t'} - x_{t'-1}) = N \sum_{t'=1}^N \big(\omega^{(t-1)(t'-1)} - \omega^{(t-1)t'}\big) x_{t'} \\
&= N \big( 1 - \omega^{(t-1)} \big) \sum_{t'=1}^N \omega^{(t-1)(t'-1)}x_{t'} = \gamma_t^{-1} (\mF \vx )_t.
\end{align*}
This shows $\hat x_t = \gamma_t (\mF \Delta \vx)_t$ for $t \neq 1$ and completes the proof.
\end{proof}
A simple bound on the modulus $|\gamma_t|$ is given by the following lemma.
\begin{lem}\label{lem:gamma}
Let $\gamma_t$ be defined by \eqref{def:gamma}. For positive integer $1< t \le N/2$, we have
\begin{equation*}
|\gamma_{N-t+2}| = |\gamma_t| \le \frac{1}{8(t-1)}\, .
\end{equation*}
\end{lem}
\begin{proof}
The equality part is obvious. For any $\theta \in (-\pi,\pi)$, we have
\begin{equation*}
| 1 - \exp(i \theta) |^2 = (1-\cos \theta)^2 + \sin^2\theta = 4\sin^2(\theta/2) = 4\sin^2(|\theta|/2)
\end{equation*}
Since $\sin \theta/\theta$ is monotone decreasing in $(0,\pi/2)$, we have $\sin(|\theta|/2) \ge \sin(\pi/2) |\theta|/(\pi/2)$, Thus,
\begin{equation*}
| 1 - \exp(i \theta) | \ge 4|\theta| / \pi \, .
\end{equation*}
Setting $\theta = -2\pi(t-1)/N $, we obtain the desired upper bound on $|\gamma_t|$.
\end{proof}
Denote $\mGamma = \diag\{\gamma_1,\ldots, \gamma_N\}$. By Lemma~\ref{lem:fourier-ibp}, for any vector $\vx \in \R^N$,
\begin{equation*}
    \mF \vx = \mGamma \mF \Delta \vx + \left( \begin{array}{c} \vx^\top \vone_N \\ 0 \\ \vdots \\ 0 \end{array} \right)\, .
\end{equation*}
Below we will assume that the generic matrix $\mA \in \R^{N \times N}$ is symmetric and satisfies $\mA \vone_N = \vzero.$ Observe that
\begin{align*}
    &\mF \mA = [\mF \mA_{:,1}, \ldots, \mF \mA_{:,N}] = \mGamma \mF [\Delta \mA_{:,1}, \ldots, \Delta \mA_{:,N}], \\
    &\mF \mA \mF^\top = \mGamma \mF (\Delta \mA) \mF^\top \mGamma
\end{align*}
where we used $\mA \vone_N = \vzero$ and $(\Delta \mA) \vone_N = \vzero$. Repeating the second equality $m$ times, we obtain
\begin{equation*}
    \hat \mA = \mF \mA \mF^\top = \mGamma^m \mF (\Delta^{(m,m)} \mA) \mF^\top \mGamma^m.
\end{equation*}
Now fix a generic nonempty index sets $\gI \subset \{1,2,\ldots,N\}$ and denote $\gJ = \{1,\ldots,N\} \setminus \gI$. Consider the block matrix form of $\hat \mA$:
\begin{equation*}
    \hat \mA = \left( \begin{array}{cc} \hat \mA_{\gI,\gI} & \hat \mA_{\gI,\gJ} \\ \hat \mA_{\gJ, \gI} & \hat \mA_{\gJ, \gJ} \end{array} \right) \, .
\end{equation*}
Using the block matrix notation, we derive
\begin{equation*}
    \norm{\hat \mA_{\gI,\gJ}}_{\op} = \norm{\mGamma_{\gI,\gI}^m \big(\mF (\Delta^{(m,m)} \mA) \mF^\top \big)_{\gI, \gJ} \mGamma_{\gJ,\gJ}^m }_{\op} \le \max_{t \in \gI, t' \in \gJ}|\gamma_t \gamma_{t'}|^m \norm{\mF (\Delta^{(m,m)} \mA) \mF^\top}_{\op}\, .
\end{equation*}
Similar inequalities hold for other three blocks. By Lemma~\ref{lem:gamma} we have $|\gamma_t| \le 1$ for all $t$. 
Adding the three inequalities that involve at least one index set $\gJ$, we get
\begin{equation*}
\left \lVert \hat \mA - \left( \begin{array}{cc} \hat \mA_{\gI,\gI} & \vzero \\ \vzero & \vzero \end{array} \right) 
\right \rVert_\op \le 3 \max_{t \in \gJ} |\gamma_t|^m \norm{\mF (\Delta^{(m,m)} \mA) \mF^\top}_{\op}\, .
\end{equation*}
Since $\mF \mF^* = N$, we have $\norm{\mF}_\op = \norm{\mF \mF^*}_\op^{1/2} = N^{1/2}$. Denoting
\begin{equation*}
    \mA^{(\res)} := N^{-2} \mF^* \Big[ \hat \mA - \left( \begin{array}{cc} \hat \mA_{\gI,\gI} & \vzero \\ \vzero & \vzero \end{array} \right) \Big] (\mF^*)^\top\, ,
\end{equation*}
we find
\begin{align}\label{ineq:res-op}
\norm{\mA^{(\res)}}_\op \le 3 \max_{t \in \gJ} |\gamma_t|^m  N^{-2} \norm{\mF}_\op^4 \norm{\Delta^{(m,m)} \mA}_\op \le 3 \max_{t \in \gJ} |\gamma_t|^m  N\norm{\Delta^{(m,m)} \mA}_{\max}
\end{align}
where the first inequality is due to Lemma~\ref{lem:gamma} and the second inequality is due to the inequality between matrix operator norm and max norm.

To finish the proof, let us make some specification: we identify $N$ with $\tilde T$ (namely $2T$), identify $\mA$ with $\tilde \mG$, and identify $\gI = \{1,\ldots,k\} \cup \{\tilde T-k+1,\ldots,\tilde T\}$. These choices satisfy the requirement for $\mA$ because $\tilde \mG$ is symmetric and it satisfies $\tilde \mG \vone_{\tilde T} = 2\mG \vone_T = \vzero$ due to the assumption $\pos_1 + \ldots + \pos_T = \vzero$.

By Lemma~\ref{lem:gamma}, $\max_{t\in \gJ}|\gamma_t| \le 1/(8k)$, so \eqref{ineq:res-op} gives
\begin{equation*}
   \norm{\mA^{(\res)}}_\op \le 6 (8k)^{-m}  T\norm{\Delta^{(m,m)} \tilde \mG}_{\max}
\end{equation*}
By the definition of $\mA^{(\res)}$ and the identities in \ref{eq:DFT},
\begin{equation*}
    \tilde \mG = \mA^{(\res)} + N^{-2} \mF^* \left( \begin{array}{cc} \hat \mA_{\gI,\gI} & \vzero \\ \vzero & \vzero \end{array} \right) (\mF^*)^\top \, .
\end{equation*}

We make the following claim.
\begin{lem}\label{lem:A-lowfreq}
There exists $\mB \in \R^{k \times k}$ such that
\begin{equation} \label{eq:A-lowfreq}
    N^{-2} \left[ \mF^* \left( \begin{array}{cc} \hat \mA_{\gI,\gI} & \vzero \\ \vzero & \vzero \end{array} \right) (\mF^*)^\top \right]_{1:T,1:T} = \mF_{\le k} \mB (\mF_{\le k} \mB)^\top 
\end{equation}
\end{lem}
While the DFT matrix is complex, the above lemma claims that the left-hand side is a Gram matrix of real low-frequency vectors. Once this lemma is proved, we can combine this lemma with \eqref{ineq:res-op} and $\mG = \tilde \mG_{1:T, 1:T}$ to obtain the desired inequality~\ref{ineq:fourier} in Theorem~\ref{thm:fourier}.

\begin{proof}[Proof of Lemma~\ref{lem:A-lowfreq}]
Recall $\omega = \exp(-2\pi i/2T)$. We further introduce some notations. Denote $\gI_1 = \{1,\ldots,k\}$ and $\gI_2 = \{T-k+1,\ldots,T\}$ so that $\gI = \gI_1 \cup \gI_2$. Let $q_t = \omega^{t-1}$ for positive integer $t$, matrix $\mQ \in \R^{T \times k}$ and matrix $\mD \in \R^{k \times k}$ be given by
\begin{equation*}
    Q_{t,s} = \Re(q_t^{s-0.5}), \qquad \text{where}~~t \le T, s \le k, \qquad \mD = \diag(q_1^{1/2}, \ldots, q_k^{1/2})
\end{equation*}
For $t,t' \in \gI$,
\begin{align*}
    \hat A_{t,t'} & = (\mF \mA \mF^T)_{t,t'} \\
    &= \sum_{s,s'=1}^T F_{t,s} G_{s,s'} F_{t',s'} + \sum_{s,s'=1}^T F_{t,2T+1-s} G_{s,s'} F_{t',s'} \\
    &+ \sum_{s,s'=1}^T F_{t,s} G_{s,s'} F_{t',2T+1-s'} + \sum_{s,s'=1}^T F_{t,2T+1-s} G_{s,s'} F_{t',2T+1-s'} \\
    &= q_t^{1/2}  q_{t'}^{1/2} \sum_{s,s'=1}^T G_{s,s'} \left( q_t^{s-0.5} q_{t'}^{s'-0.5} + q_t^{-s+0.5} q_{t'}^{s'-0.5} + q_t^{s-0.5} q_{t'}^{-s'+0.5} + q_t^{-s+0.5} q_{t'}^{-s'+0.5}   \right) \\
    &= 4q_t^{1/2}  q_{t'}^{1/2} \sum_{s,s'=1}^T G_{s,s'} \Re(q_t^{s-1/2}) \Re(q_{t'}^{s'-1/2})\,.
\end{align*}
If $t,t' \in \gI_1$, the above equality leads to 
\begin{equation*}
    \hat \mA_{\gI_1, \gI_1} = 4\mD \mQ^\top \mG \mQ \mD;
\end{equation*}
and more generally $\hat \mA_{\gI, \gI}$ is given by symmetrically extending $\hat \mA_{\gI_1, \gI_1}$ as in the definition of $\tilde \mG$. Since $4\mQ^\top \mG \mQ$ is a PSD, we can find $\mB_0 \in \R^{k \times k}$ such that
\begin{equation*}
    4\mQ^\top \mG \mQ = \mB_0 \mB_0^\top \,.
\end{equation*}
We want to simplify $\mF^*_{:,\gI} \hat \mA_{\gI, \gI} \mF^*_{\gI,:}$, namely
\begin{align}\label{expr:four-term}
    \mF^*_{:,\gI_1} \hat \mA_{\gI_1, \gI_1} \mF^*_{\gI_1,:} +  \mF^*_{:,\gI_1} \hat \mA_{\gI_1, \gI_2} \mF^*_{\gI_2,:} + \mF^*_{:,\gI_2} \hat \mA_{\gI_2, \gI_1} \mF^*_{\gI_1,:} + 
    \mF^*_{:,\gI_2} \hat \mA_{\gI_2, \gI_2} \mF^*_{\gI_2,:}
\end{align}
For any $t,t' \in \gI_1$, 
\begin{align*}
    (\mF^*)_{t,\gI_1} \hat \mA_{\gI_1, \gI_1} (\mF^*)_{\gI_1,t'} = (\mF^*)_{t,\gI_1} \mD \mB_0 \mB_0^\top  \mD (\mF^*)_{\gI_1,t'}
\end{align*}
and similar equations hold for other three cases. Observe that 
\begin{align*}
  (\mF^*)_{t,\gI_1} \mD = (\bar q_t^{1-0.5},  \ldots, \bar q_t^{k-0.5} ), \qquad (\mF^*)_{t,\gI_2} \mD = (\bar q_t^{2T-0.5},  \ldots, \bar q_t^{2T-k+0.5} )\,.
\end{align*}
When we add the four terms in \ref{expr:four-term}, the imaginary part cancels out. Thus, 
\begin{align*}
    \mF^*_{t,\gI} \hat \mA_{\gI, \gI} \mF^*_{\gI,t'} = 4 \left( \Re(\bar q_t^{1-0.5}),  \ldots, \Re(\bar q_t^{k-0.5}) \right) \mB_0 \mB_0^\top \left( \Re(\bar q_t^{1-0.5}),  \ldots, \Re(\bar q_t^{k-0.5}) \right)^\top\,.
\end{align*}
Note that $\Re(\bar q_t^{s-0.5}) = \cos(\pi(t-1)(s-0.5)/T) = (\mF_{\le k})_{t,s}$. Expressing $\mF^*_{:,\gI} \hat \mA_{\gI, \gI} \mF^*_{\gI,:}$ in the matrix form and denote $\mB = \mB_0 / N$, we find that \eqref{eq:A-lowfreq} holds.
\end{proof}

\subsection{Proof of Theorem~\ref{thm:incoh}}

First we note that 
\begin{equation}\label{eq:firstKdecomop}
    K_{\mW}(\vx^q, \vx^k) = K_{\mW_{11}}(\vx^q, \vx^k) K_{\mW_{12}}(\vx^q, \vx^k) K_{\mW_{21}}(\vx^q, \vx^k) K_{\mW_{22}}(\vx^q, \vx^k)\, .
\end{equation}
We will prove that 
\begin{equation}\label{eq:W11}
    K_{\mW_{11}}(\vx^q, \vx^k) = \big(1 + O(\incoh) \big) \cdot K_{\mW_{11}}(\vc^q, \vc^k)\, .
\end{equation}
To prove this, it suffices to show that
\begin{equation}\label{eq:maxK}
    \max\big\{ K_{\mW_{11}}(\vc^q, \vt^k), K_{\mW_{11}}(\vt^q, \vc^k), K_{\mW_{11}}(\vt^q, \vt^k) \big\} = 1 + O(\incoh) \, .
\end{equation}
We decompose $\mW_{11}$ as in \eqref{eq:mW} and find
\begin{align*}
    \log \left( K_{\mW_{11}}(\vc^q, \vt^k) \right) = (\vc^q)^\top \mW_{11} \vt^k = \sum_{k=1}^s a_k (\vu_k^\top \vc^q)(\vv_k^\top \vt^k)\,.
\end{align*}
By mutual incoherence, $|\vv_k^\top \vt^k| \le \incoh$ since $\vv_k \in \gB_1$ and $\vt^k \in \gB_2$; and trivially $|\vu_k^\top \vc^k| \le 1$, so
\begin{align*}
    \left|\log \left( K_{\mW_{11}}(\vc^q, \vt^k) \right) \right| \le \sum_{k=1}^s |\vu_k^\top \vc^k| \cdot |\vv_k^\top \vt^q| \le s\cdot \incoh\, .
\end{align*}
Since by assumption $s=O(1)$ and $\exp(\incoh) = 1+O(\incoh)$, we derive
\begin{equation*}
    K_{\mW_{11}}(\vc^q, \vt^k) = 1 + O(\incoh) \, .
\end{equation*}
The other two terms in \eqref{eq:maxK} follow a similar argument and thus are all bounded by $1 + O(\incoh)$. We can prove similarly that
\begin{align*}
    &K_{\mW_{12}}(\vx^q, \vx^k) = \big(1 + O(\incoh) \big) \cdot K_{\mW_{12}}(\vc^q, \vt^k)\, , \\
    &K_{\mW_{21}}(\vx^q, \vx^k) = \big(1 + O(\incoh) \big) \cdot K_{\mW_{21}}(\vt^q, \vc^k)\, ,\\
    &K_{\mW_{21}}(\vx^q, \vx^k) = \big(1 + O(\incoh) \big) \cdot K_{\mW_{22}}(\vt^q, \vt^k)\, .
\end{align*}
and together with \eqref{eq:W11} and \eqref{eq:firstKdecomop}, this leads to the desired \eqref{eq:incoh}.

Below we prove the ``moreover'' part. By standard properties of independent subgaussian random variables \citep[Sect.~2]{vershynin2018high}, $(\vc^q)^\top \rmZ_{11} \vt^k$ is still a subgaussian random variable, and with probability at least $1 - O(\exp(-\incoh^2\cdot d))$, for certain constant $C>0$,
\begin{align*}
    \left|\log \left( K_{\mW_{11}+\rmZ/\sqrt{d}}(\vc^q, \vt^k) \right) \right| \le  s\cdot \incoh\, + C\incoh = O(\incoh) .
\end{align*}
Similar high-probability bounds hold for other terms. By the union bound over all possible choice of vectors in $\gB_1^0$ and $\gB_2^0$, we arrive at our claim.